\documentclass[10pt,journal,compsoc]{IEEEtran}

\usepackage[nocompress]{cite}
\usepackage{url}
\usepackage{multirow}
\usepackage{graphicx}
\usepackage{epstopdf}
\usepackage{booktabs}
\usepackage{amstext}
\usepackage{amsthm}
\usepackage{amsmath}
\theoremstyle{plain}

\newtheorem{theorem}{Theorem}
\newtheorem{lemma}{Lemma}

\theoremstyle{definition}
\newtheorem*{prf}{Proof}

\usepackage{enumitem}
\usepackage{amssymb}
\usepackage{algorithm}
\usepackage{algpseudocode}

\usepackage{makecell}
\usepackage{array}
\newcommand{\PreserveBackslash}[1]{\let\temp=\\#1\let\\=\temp}
\newcolumntype{C}[1]{>{\PreserveBackslash\centering}p{#1}}
\newcolumntype{R}[1]{>{\PreserveBackslash\raggedleft}p{#1}}
\newcolumntype{L}[1]{>{\PreserveBackslash\raggedright}p{#1}}
\usepackage{color}
\usepackage{subfig}

\begin{document}

\title{Learnable Weighting of Intra-Attribute Distances for Categorical Data Clustering with \\ Nominal and Ordinal Attributes}

\author{Yiqun~Zhang,~\IEEEmembership{Member,~IEEE}
        and~Yiu-ming~Cheung,~\IEEEmembership{Fellow,~IEEE}

}
\markboth{IEEE TRANSACTIONS ON PATTERN ANALYSIS AND MACHINE INTELLIGENCE, 2022}%
{Zhang \MakeLowercase{\textit{et al.}}: Bare Demo of IEEEtran.cls for Computer Society Journals}

\IEEEtitleabstractindextext{%
\begin{abstract}
The success of categorical data clustering generally much relies on the distance metric that measures the dissimilarity degree between two objects. However, most of the existing clustering methods treat the two categorical subtypes, i.e. nominal and ordinal attributes, in the same way when calculating the dissimilarity without considering the relative order information of the ordinal values. Moreover, there would exist interdependence among the nominal and ordinal attributes, which is worth exploring for indicating the dissimilarity. This paper will therefore study the intrinsic difference and connection of nominal and ordinal attribute values from a perspective akin to the graph. Accordingly, we propose a novel distance metric to measure the intra-attribute distances of nominal and ordinal attributes in a unified way, meanwhile preserving the order relationship among ordinal values. Subsequently, we propose a new clustering algorithm to make the learning of intra-attribute distance weights and partitions of data objects into a single learning paradigm rather than two separate steps, whereby circumventing a suboptimal solution. Experiments show the efficacy of the proposed algorithm in comparison with the existing counterparts.
\end{abstract}

}

\maketitle

\IEEEdisplaynontitleabstractindextext

\IEEEpeerreviewmaketitle

\IEEEraisesectionheading{\section{Introduction}\label{sct:introduction}}

\IEEEPARstart{W}{idespread} categorical data can be easily collected from questionnaires, medical scales, scoring systems, and so on \cite{intro1}. As one of the most widely used machine learning and pattern recognition techniques, clustering that partitions data objects into homogeneous groups in unsupervised environment \cite{kms,over11} has been commonly adopted for the analysis of categorical data \cite{intro5,intro11}. In order to better discover homogeneous clusters, weighting attributes according to their importance to the clustering task \cite{over1} is adopted by many existing clustering algorithms \cite{wkm,wkm2,ewkm,mwkm,woc}. Since weighting an attribute is equivalent to uniformly weighting all the intra-attribute distances measured on this attribute, these algorithms are actually based on the hypothesis that all the intra-attribute distances are well defined, which is reasonable for numerical data with well-defined distance measure \cite{over10}. However, for categorical data whose distance measure is generally not well-defined, uniformly weighting the intra-attribute distances is surely unreasonable \cite{intro10}. To solve this problem, most existing methods focus on exploring appropriate distance measures \cite{intro6,intro7} and attribute weighting mechanisms \cite{woc}.

Successful attempts in exploring appropriate distance measures include Lin's \cite{lsm} similarity measure, coupled \cite{r1compare2} similarity metric, association-based \cite{abdm}, Ahmad's \cite{adm}, context-based \cite{cbdmconf,cbdmjournal}, and Jia's \cite{jdm} distance metrics. The above-mentioned measures define intra-attribute distances according to the possible value statistics, e.g., the occurrence frequencies and conditional occurrence probabilities. Lin's measure computes the cumulative entropy of a range of ordered possible values (i.e., the adjacent possible values \{good, neutral, bad\} of an ordinal attribute with possible values \{very-good, good, neutral, bad, very-bad\}) to indicate the corresponding intra-attribute distance (i.e., the distance between good and bad) with preserving the order relationship, which is suitable for the distance measurement of ordinal data. The others define intra-attribute distances according to the context information reflected by conditional probability distributions between interdependent attributes, which works well for nominal data. In recent years, more powerful representation-based methods including structure-based \cite{r1compare1}, coupled \cite{r1cluster2,r1compare3}, and heterogeneous coupling \cite{r1compare4} representations, have been proposed to represent categorical data by embedding more informative and complex relationships existing in the level of values, attributes, and objects, so as to achieve a more reasonable distance measurement. Unfortunately, they still work well for nominal data only.

In summary, all the above mentioned measures are proposed without considering a very common situation that real categorical data are usually composed of a mixture of nominal and ordinal attributes \cite{intro2,intro3}. As the fragment of medical scale data set shown in Table~\ref{tb:example_mix_data}, the values of ordinal Attribute 1 stand for the degrees of lymph enlargement, the values of nominal Attribute 2 stand for the special form of lymph, and the values of the Class attribute indicate the diagnosis results, which are the desired true cluster labels in cluster analysis.
\begin{table}[t]
\caption{Fragment of Lymphography data set.}
\label{tb:example_mix_data}
\centering
\begin{tabular}{c|cc|c}
\toprule
\multirow{2}{*}{No.}      & Attribute 1       & Attribute 2   & Class\\
&(enlarge) & (form)     & (diagnosis) \\
\midrule
1        & $\uparrow$     & non-special  & normal    \\
2        & $\uparrow$     & vesicles     & fibrosis    \\
3        & $\uparrow\uparrow$     & vesicles     & fibrosis    \\
4        & $\uparrow\uparrow\uparrow$     & chalices     & metastases   \\
5        & $\uparrow\uparrow\uparrow$     & chalices     & malign       \\
6        & $\uparrow\uparrow\uparrow\uparrow$     & vesicles     & malign    \\
\bottomrule
\end{tabular}
\end{table}
In medicine, it is generally believed that the severity of fibrosis, metastasis, and malign increases in sequence. Apparently, if we treat the ordered values of Attribute 1 as nominal values, information provided by the monotonic relationship between the values of Attribute 1 and the true class labels will be lost \cite{ex6}, which will directly affect the clustering accuracy. Moreover, there also exists an awkward gap between the cluster information provided by nominal and ordinal attributes, because the values of an ordinal attribute contain the relative ordering information, but the values of a nominal one do not. Hence, to avoid the loss of important information, entropy-based distance metrics \cite{ebdmjournal,udm} have been proposed to quantify intra-attribute distances of nominal and ordinal attributes as information entropy \cite{over3} in a unified way. However, they have not established an essential connection between nominal and ordinal attributes for data clustering.

As for attribute weighting mechanism, most efforts have tried to weight attributes for each cluster, which is called subspace clustering. Typical subspace approaches include \cite{wkm,wkm2,mwkm,woc}, which learn the different weight combinations of attributes for each cluster to explore more appropriate subspaces for gathering homogeneous data objects. Nevertheless, they uniformly weight all intra-attribute distances measured on the same attribute, which still makes these approaches incompetent in adapting the contributions of different intra-attribute distances to search for more appropriate clustering results. Most recently, a distance weighting-based clustering algorithm \cite{dlc} has been proposed to learn the weights of intra-attribute distances automatically during clustering. This algorithm has remarkable performance on ordinal data sets, but it relies on the order relationship among attribute values for learning the distance weights, which makes it applicable to ordinal data only. To the best of our knowledge, clustering algorithm that can learn the weights of intra-attribute distances for categorical data with nominal and ordinal attributes has yet to be proposed.

In this paper, we will propose a new clustering method composed of a novel distance definition and an automatic distance weighting mechanism for any-type categorical data clustering, i.e., clustering data composed of any combination of nominal and ordinal attributes. Specifically, we study the intrinsic difference and connection of nominal and ordinal attributes, and convert each possible value of nominal attributes, e.g., ``vesicles'' of Attribute 2 as shown in Table~\ref{tb:example_mix_data}, into a Boolean attribute with two possible values ``vesicles'' and ``not vesicles''. Such Boolean attribute is a special case of ordinal attribute, i.e., an ordinal attribute with two extreme degrees ``vesicles'' and ``not vesicles''. Thus, the heterogeneous clustering information provided by nominal and ordinal attributes becomes homogeneous information provided by ordinal attributes. On this basis, the information provided by interdependent attributes in three cases (i.e., (i) both attributes are nominal, (ii) both attributes are ordinal, and (iii) one is nominal and the other is ordinal) is utilized to measure intra-attribute distances of nominal and ordinal attributes in a unified way. Since the defined distances are not connected to a certain clustering task, we also propose a novel intra-attribute distance weighting mechanism to learn the distance weights iteratively based on the present data partition result to search for better clustering results. The proposed distance definition and weighting mechanism are complementary to each other in clustering. It turns out that the clustering algorithm utilizing them is competent for the cluster analysis of any-type categorical data. The main contributions of this paper are summarized below:
\begin{itemize}
\item Inherent connection of nominal and ordinal attributes is studied, and a novel measure suitable for intra-attribute distance measurement of any-type categorical data clustering is proposed accordingly.
\item An intra-attribute distance weighting mechanism that iteratively updates the distance weights to search for better data partitions, if any, is designed to make the measured intra-attribute distances learnable.
\item A new categorical data clustering algorithm is presented by utilizing the learnable distance measure. This algorithm is parameter free and has superior clustering performance on any-type categorical data.
\end{itemize}

The remainder of this paper is organized as follows. Section~\ref{sct:related} reviews the related works. Section~\ref{sct:preliminary} formulates the research problems. A design of homogeneous distance metric is proposed in Section~\ref{sct:dist}. Then, Section~\ref{sct:clustering} introduces a new clustering algorithm with the novel distance weighting mechanism as the core. Experimental results are given in Section~\ref{sct:experiments}. Finally, we draw a conclusion in Section~\ref{sct:conclusion}.

\section{Related Work}\label{sct:related}

This section makes an overview of the existing related works on categorical data clustering.

\subsection{Distance Measure}\label{subsct:related_measure}

The distance measures for categorical data clustering can be generally categorized as the direct, context-based, and representation-based ones. The simplest direct measure \cite{hdm} directly assigns distances 0 and 1 to identical and different intra-attribute values, respectively. The other direct measures \cite{lsm,ebdmconf,ebdmjournal} compute the intra-attribute distance between two possible values according to their occurrence frequencies. Direct measures are easy to use and have demonstrated great computational efficiency because their computation does not involve parameter selection, context information extraction, iterative learning, etc. However, since the valuable information provided by the correlated attributes is totally ignored, intra-attribute distances defined by them are not always reasonable in indicating the real dissimilarity degrees.

In contrast, the context-based measures \cite{r1compare2,abdm,adm,cbdmconf,cbdmjournal,jdm,udm} compute the distance between two intra-attribute values based on the context information, i.e., the statistical information provided by the other attributes that are correlated with the target one. In general, these measures outperform the direct ones, but their performance dependents more on the interdependence of attributes. For the data composed of independent attributes, some indirect measures \cite{abdm,adm,cbdmconf,cbdmjournal} that are based on the sole information provided by the interdependent attributes would even fail for distance measurement. Among all the above-mentioned indirect and context-based measures, the two measures \cite{ebdmjournal,udm} that unify the distance concept of nominal and ordinal attributes as the information divergence to avoid information loss are suitable for any-type categorical data clustering. Nevertheless, they only provide scale-level distance unification, but have yet to consider the intrinsic connection between nominal and ordinal attributes.

The representation-based distance measures encode categorical values into numerical ones, and then the advanced distance measures and clustering algorithms proposed for numerical data can be utilized. In many practical application scenarios, the encoding is performed by domain experts, which makes the performance sensitive to the prior knowledge. Further, for large-scale, high-dimensional, and multi-variate categorical data, the encoding process is a laborious and non-trivial task. A commonly adopted way to circumvent these issues is to simply encode each possible value of nominal attributes into a binary-valued numerical attribute and the ordered possible values of each ordinal attribute into consecutive integers, which is called simple coding. It turns out that simple coding is applicable to any-type categorical data. Nevertheless, since it ignores the original statistical information of possible values, and it assigns the identical distance to different possible value pairs, empirical studies in \cite{dlc} have shown that its performance is generally worse than the measures specially designed for categorical data. Recently, representation learning methods \cite{r1compare1,r1compare3,r1compare4} have been proposed for automatically encoding categorical data in unsupervised environment. The one called SBC \cite{r1compare1} reconstructs the original data set according to the inter-object dissimilarities. CDE in \cite{r1compare3} encodes the original data set by performing $k$-means clustering and PCA on intra- and inter-attribute couplings. The newly proposed UNTIE \cite{r1compare4} represents data set by using more types of couplings learned in multiple kernel spaces, and achieves superior clustering performance. However, all the above-mentioned representation learning methods are actually designed for nominal data only, and their performance somewhat depends on the non-trivial selection of parameters or kernel functions.

\subsection{Clustering Algorithm}\label{subsct:related_algorithm}

From the perspective of attribute weighting, the existing categorical data clustering algorithms can be roughly categorized as the non-attribute-weighting and attribute-weighting ones, respectively. As a non-attribute-weighting algorithm, the conventional $k$-modes \cite{kmd} adopts Hamming distance \cite{hdm} as a distance measure to compute the distance between data objects and the $k$ modes. Based on the object-mode distances, it iteratively searches for better partitions of data set. Furthermore, some of its variants also focus on improving its robustness and scalability \cite{kpt}\cite{oc}. In addition, clustering algorithm adopting entropy as a measure \cite{ecc} has been proposed in the literature. It computes the entropy value of the present partition after moving an object into a cluster, and performs cluster analysis by searching for the partition with the minimum entropy value \cite{over4}. In general, all the above-mentioned algorithms assume that the attributes are of identical importance for clustering tasks, which is, however, not always true in practice.

In the literature, an attribute weighting-based categorical data clustering algorithm \cite{wkm} has been proposed provided that the attributes are of different importance. It assigns different weights to the attributes according to their contributions in forming more compact clusters. That is, if the total distance between data objects and their clusters measured on a certain attribute is low, it indicates that this attribute contributes more than the others in forming the clusters with similar objects. Subsequently, a higher weight is thus assigned to this attribute in the next iteration to search for more compact clusters. Nevertheless, this weighting mechanism
finds only a certain attributes' subset that is important to a certain subset of clusters, which is evidently incompetent in a more complex case. Therefore, subspace clustering algorithms \cite{wkm2,ecc,mwkm,woc} that weight each attribute according to its contribution in forming each certain cluster have been proposed.

In general, weighting an attribute is equivalent to uniformly weighting all the distances measured on it. Thus, all the above-mentioned attribute weighting-based algorithms actually assume that the distance measure can accurately indicate the intra-attribute distances. If the adopted distance measure is not appropriately defined, uniformly weighting the intra-attribute distances measured by them will just bring more irrationality into the clustering process. Therefore, the most recently proposed clustering algorithm \cite{dlc} addresses this issue by iteratively weighting the importance of intra-attribute distances according to the present partition to search for more appropriate clustering results of the data set. Unfortunately, distance weighting of this algorithm relies on the order relationship among intra-attribute values, which makes it only applicable to the categorical data sets composed of ordinal attributes.

\section{Problem Statement}\label{sct:preliminary}

We formulate the problem of distance weighting-based clustering of categorical data in this section. Table~\ref{tb:notation} lists the styles of notations used in this paper.
\begin{table}[t]
\caption{Style of notations and explanation of symbols.}
\label{tb:notation}
\centering
\begin{tabular}{ll}
\toprule
Notation (example) & Style \\
\midrule
Attribute index (e.g. $A^r$) & Superscript \\
Value note (e.g. $d^{(\text{ord})}$) & Superscript with parentheses \\
Function (e.g. $\text{dist}(\cdot,\cdot)$) & Parentheses \\
Space (e.g. $\mathcal{R}_0^+$) & Uppercase, calligraphic font \\
Vector (e.g. $\textbf{p}_l^r$) & Lowercase, bold font \\
Matrix (e.g. $\textbf{Q}$) & Uppercase, bold font \\
\bottomrule
\toprule
Symbol (example) & Explanation of example \\
\midrule
$\emptyset$ (e.g. $A^\text{(ord)}=\emptyset$) & $A^\text{(ord)}$ is an empty set \\
$\top$ (e.g. $[x_i^1,x_i^2,...,x_i^d]^\top$) & Transpose of $[x_i^1,x_i^2,...,x_i^d]$ \\
$\succ$ (e.g. $o^r_1\succ o^r_2$) & $o^r_1$ ranks higher than $o^r_2$ \\
$\neg$ (e.g. $\neg o^s_g$) & $A^s$'s possible values excluding $o^s_g$ \\
\bottomrule
\end{tabular}
\end{table}
A categorical data set $S$ can be represented as a tuple $S=<X,A,O>$, where $X=\{\textbf{x}_i|i\in N_X\}$ is the object set with $n$ elements, and $N_X=\{1,2,...,n\}$ is the index set of $X$. For attribute set $A$ composed of $d$ attributes, we assume that the former $d^{(\text{ord})}$ attributes are ordinal and the latter $d^{(\text{nom})}$ attributes are nominal for convenience without loss of generality, and we have $d^{(\text{ord})}+d^{(\text{nom})}=d$. Formally, $A^{(\text{ord})}=\{A^r|r\in N_A^{(\text{ord})}\}$ is the ordinal attribute set, $A^{(\text{nom})}=\{A^s|s\in N_A^{(\text{nom})}\}$ is the nominal attribute set, $N_A^{(\text{ord})}=\{1,2,...,d^{(\text{ord})}\}$ and $N_A^{(\text{nom})}=\{d^{(\text{ord})}+1,d^{(\text{ord})}+2,...,d\}$ are the index sets of $A^{(\text{ord})}$ and $A^{(\text{nom})}$, respectively. $A=A^{(\text{ord})}\cup A^{(\text{nom})}$ is the complete attribute set, and $N_A=N_A^{(\text{ord})}\cup N_A^{(\text{nom})}$ is the complete index set of $A$. Accordingly, three types of categorical data can be distinguished by:
\begin{equation}\label{eq:data_type}
\resizebox{0.91\linewidth}{!}{$
\text{datatype}(S)=\left\{
\begin{array}{lll}
\text{mixed},  & A^{(\text{ord})}\neq\emptyset,\ \ A^{(\text{nom})}\neq\emptyset\\
\text{ordinal},& A^{(\text{ord})}\neq\emptyset,\ \ A^{(\text{nom})}=\emptyset\\
\text{nominal},& A^{(\text{ord})}=\emptyset,\ \ A^{(\text{nom})}\neq\emptyset.
\end{array}
\right.
$}
\end{equation}
Hereinafter, a categorical data set composed of a mixture of ordinal and nominal attributes, pure ordinal attributes, and pure nominal attributes is called mixed, ordinal, and nominal data set, respectively. $O^r=\{o_m^r|m\in N_O^r\}$ is the set of $v^r$ possible values of attribute $A^r$, and $N_O^r=\{1,2,...,v^r\}$ is the index set of $A^r$'s possible values. The $i$th object of $X$ is represented as $\textbf{x}_i=[x_i^1,x_i^2,...,x_i^d]^\top$ with $x_i^r\in O^r$, $r\in N_A$. If $A^r$ is an ordinal attribute (i.e. $r\leq d^{(\text{ord})}$), its possible values satisfy $o_1^r\succ o_2^r\succ ... \succ o_{v^r}^r$ where the symbol ``$\succ$'' indicates that the values on its left are rank higher than the values on its right.

In crisp partitional clustering task, $X$ is partitioned into $k$ clusters, which can be represented as a cluster set $C=\{C_l|l\in N_C\}$ with $N_C=\{1,2,...,k\}$. Accordingly, $X$ can be represented as a collection of $k$ disjoint subsets $X=\bigcup_{l=1}^k X_{C_l}$ where $X_{C_l}$ is the object set corresponding to the $l$th cluster. The $k$ clusters are represented by their corresponding statistical information $P=\{P_l|l\in N_C\}$ where $P_l=\{\textbf{p}_l^r|r\in N_A\}$ is the statistical information of $C_l$ and $\textbf{p}_l^r=[p_{l1}^r,p_{l2}^r,...,p_{lv^r}^r]^\top$ is the probability distribution of the $r$th values of the objects in $C_l$. Values of $P$ are dependent on $\textbf{Q}$, which is an $n\times k$ matrix indicating the partition of $X$. The $(i,l)$th entry of $\textbf{Q}$ is denoted as $q_{il}$. If $\textbf{x}_i$ belongs to $C_l$, we have $q_{il}=1$, otherwise, $q_{il}=0$. To learn the importance of intra-attribute distances, we solve the clustering problem in a distance weighting framework. The weights of intra-attribute distances are denoted as a set of matrices $W=\{\textbf{W}^r|r\in N_A\}$ where $\textbf{W}^r$ is a $v^r\times v^r$ symmetric matrix storing the weights of intra-attribute distances of $A^r$. The $(m,h)$th entry of $\textbf{W}^r$ is denoted as $w_{mh}^r$, which represents the weight of the distance between possible values $o_m^r$ and $o_h^r$. The clustering problem can be formulated as minimizing the objective function:
\begin{equation}\label{eq:objective}
Z(\textbf{Q},P,W)=\sum_{i=1}^n\sum_{l=1}^kq_{il}\text{dist}(\textbf{x}_i,C_l)
\end{equation}
$$
s.t.\
\left\{
\begin{array}{ll}
\sum_{l=1}^kq_{il}=1,\ \ q_{il}\in\{0,1\},& i\in N_X,\\
\sum_{r=1}^d\sum_{m=1}^{v^r-1}\sum_{h=m+1}^{v^r}w^r_{mh}=1,& w_{mh}^r\in \mathcal{R}^+_0.\\
\end{array}
\right.
$$

The object-cluster distance $\text{dist}(\textbf{x}_i,C_l)$ is defined as
\begin{equation}\label{eq:dist_oc}
\text{dist}(\textbf{x}_i,C_l)=\sum_{r=1}^d\text{dist}^r(\textbf{x}_i,C_l),
\end{equation}
and $\text{dist}^r(\textbf{x}_i,C_l)$ is the object-cluster distance measured on attribute $A^r$. If $x^r_i=o^r_m$, $\text{dist}^r(\textbf{x}_i,C_l)$ can be written as
\begin{equation}\label{eq:dist_oca}
\text{dist}^r(\textbf{x}_i,C_l)=\sum_{h=1}^{v^r}w^r_{mh}\text{dist}^r(o^r_m,o^r_h)p^r_{lh},
\end{equation}
and the intra-attribute distance $\text{dist}^r(o^r_m,o^r_h)$ is defined as
\begin{equation}\label{eq:dist_basic}
\text{dist}^r(o^r_m,o^r_h)=\frac{1}{d}\sum_{s=1}^d\text{dist}^{rs}(o^r_m,o^r_h)
\end{equation}
where the superscript ``$rs$'' of $\text{dist}^{rs}(o^r_m,o^r_h)$ indicates that this is the intra-attribute distance between $A^r$'s possible values with respect to $A^s$. We define distance in the form of Eq.~(\ref{eq:dist_basic}) in order to exploit context information provided by interdependent attributes for distance measurement as most categorical data distance measures do \cite{abdm,adm,cbdmjournal,jdm,r1compare2,udm}. The exact definition of $\text{dist}^{rs}(o^r_m,o^r_h)$ will be given in Section~\ref{subsct:dist}.

Similar to most existing $k$-modes-type algorithms, the minimization problem of Eq.~(\ref{eq:objective}) can be solved by iteratively computing one variable and fixing the others. Since the values of $P$ are completely dependent on the values of $\textbf{Q}$, we can iteratively solve the following two problems:

\begin{itemize}
\item\textbf{P. 1}: Fix $W=\hat{W}$ and $P=\hat{P}$, solve the reduced problem $Z(\textbf{Q},\hat{P},\hat{W})$, update $P$ according to $\textbf{Q}$;
\item\textbf{P. 2}: Fix $\textbf{Q}=\hat{\textbf{Q}}$ and $P=\hat{P}$, solve the reduced problem $Z(\hat{\textbf{Q}},\hat{P},W)$.
\end{itemize}

\section{Homogeneous Distance Measurement}\label{sct:dist}

For cluster analysis, the adopted distance measure usually dominates clustering performance. In this section, we study the differences and commonalities of ordinal and nominal attributes, and then propose a homogeneous intra-attribute distance definition for them.

\subsection{Attribute Structure}\label{subsct:structure}

\begin{figure}[t]
\newcommand{\mylwd}{1}
\newcommand{\mydmwd}{3in}
  \centerline{\includegraphics[width=2.6in]{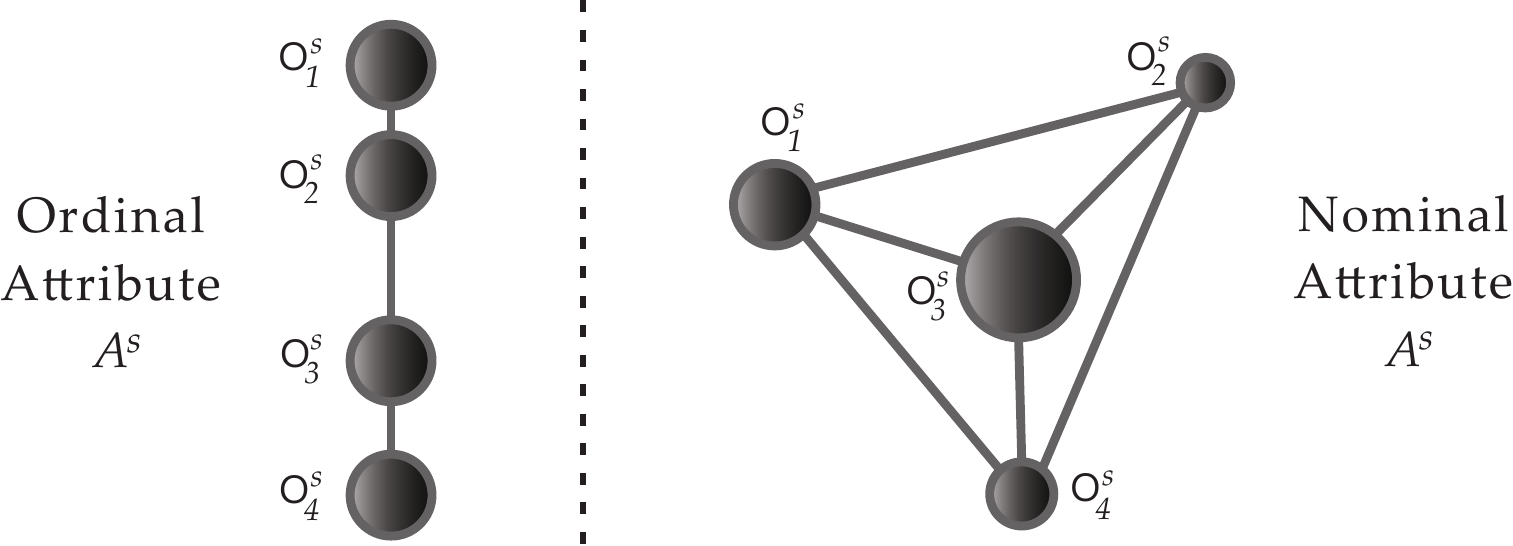}}
\caption{Structural difference between ordinal and nominal attributes from the perspective of graph. The black nodes stand for possible values and the edges reflect the spatial relationships among possible values.}
\label{fig:diff}
\end{figure}

We firstly discuss the difference between ordinal and nominal attributes. As shown in Fig.~\ref{fig:diff}, if we treat the intra-attribute possible values as nodes connected by edges, since nodes of an ordinal attribute are naturally ordered, one node cannot be reached along the edges from another non-adjacent node without crossing its adjacent node, while for a nominal attribute, a node can be directly reached along an edge from any node without involving such ``crossing''. We construct graphs for studying the heterogeneity between ordinal and nominal attributes because graph is effective in modeling complex relationships between nodes \cite{r1survey2,r1survey1}, and has been successfully applied to different machine learning tasks, such as sketch synthesis \cite{r2survey1}, item recommendation \cite{r1bandit1}, object retrieval \cite{r2survey2}, etc. It can be seen according to Fig.~\ref{fig:diff} that the structure of ordinal attribute is line-like while the structure of nominal attribute is net-like. These structures are consistent with the relationships among intra-attribute possible values of ordinal and nominal attributes from the practical point of view. For example, if we compare two choices, i.e., bad and very-good, of the review result regarding the novelty of a manuscript with the five choices $\{\text{very-good},\text{good},\text{neural},\text{bad},\text{very-bad}\}$. We will not skip neutral and good to directly compare bad and very-good, because all the choices are clearly ordered. In contrast, if we compare two choices that belong to a choice set without such order relationship, we will directly compare the two choices without involving the other choices. It is obvious that the structures of ordinal and nominal attributes are heterogeneous, which makes their intra-attribute distances difficult to be defined in a homogeneous way.

\subsection{Homogeneous Learning}\label{subsct:homo}

\begin{figure}[t]
\newcommand{\mylwd}{1}
\newcommand{\mydmwd}{3in}
  \centerline{\includegraphics[width=2.6in]{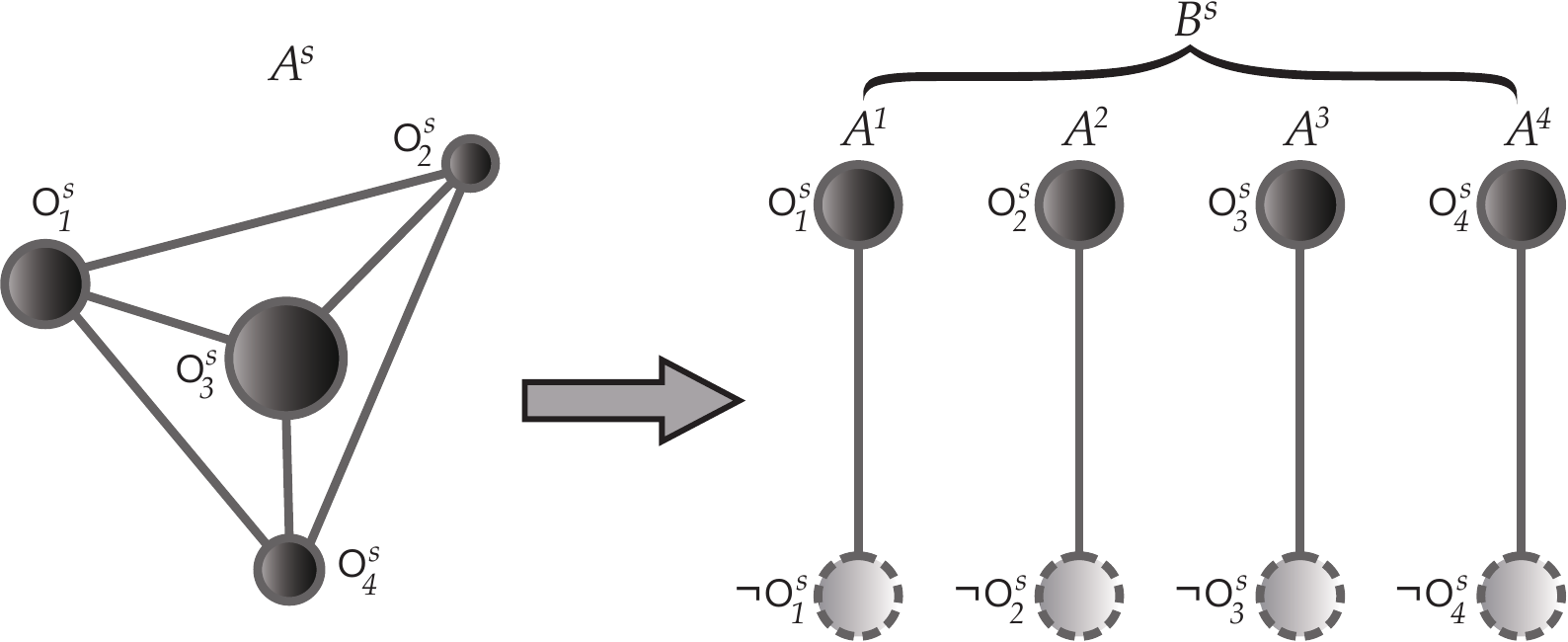}}
\caption{Converting a nominal attribute $A^s$ into a set of ordinal attributes $B^s$: Each nominal possible value $o^s_g$ is converted into an ordinal attribute $A^g$ with two ordered possible values $o^s_g$ and $\neg o^s_g$.}
\label{fig:convert}
\end{figure}

For mixed categorical data, there are two cases for Eq.~(\ref{eq:dist_basic}): 1) $A^s\in A^{(\text{ord})}$, and 2) $A^s\in A^{(\text{nom})}$. Since the possible values of an ordinal attribute represent the different degrees of a concept while the possible values of a nominal attribute represent different concepts, we convert possible values of a nominal attribute into ordinal attributes as shown in Fig.~\ref{fig:convert} so that the original nominal attribute becomes homogeneous with ordinal attributes. Specifically, for $A^s\in A^{(\text{nom})}$ with $v^s$ possible values, we convert it into a set of $v^s$ ordinal attributes
\begin{equation}\label{eq:expand}
B^s=\{A^g|g\in N_O^s\}
\end{equation}
where $A^g$ is a newly generated ordinal attribute corresponding to the possible value $o_g^s$ of $A^s$. Each $A^g$ has two possible values $o_1^g=o_g^s$ and $o_2^g=\neg o_g^s$ where $v^g=2$ and $o_1^g\succ o_2^g$. Here, $\neg o_g^s$ stands for all the possible values of $A^s$ except $o_g^s$. Each $A^g$ can be viewed as a special case of ordinal attribute, in which there are only two possible values indicating two extreme degrees, i.e., ``is $o_g^s$'' and ``is not $o_g^s$''. In this way, all the nominal attributes can be converted into ordinal attributes, and the intra-attribute distances can then be measured according to the same type of information provided by the attributes.

\subsection{Design of Proposed Distance Metric}\label{subsct:dist}

The distance between two possible values (e.g., ${o}^r_m$ and ${o}^r_h$ of attribute $A^r$) with respect to another attribute (e.g. $A^s$) is defined in this part. Before presenting the details of this distance definition, let us first define the conditional probability distribution of an attribute (e.g., $A^s$) with respect to a possible value (e.g., $o_m^r$), which can be written as
\begin{equation}\label{eq:distribution_ord}
\textbf{u}_m^{rs}=[p(o_1^s|o_m^r),p(o_2^s|o_m^r),...,p(o_{v^s}^s|o_m^r)]^\top
\end{equation}
where $p(o_g^s|o_m^r)$ is the conditional probability of $ {o}_g^s$ with respect to $ {o}_m^r$ following Bayes' theorem:
\begin{equation}\label{eq:cond_prob_ord}
p( {o}_g^s| {o}_m^r)=\frac{\text{card}(X_g^s\cap X_m^r)}{\text{card}(X_m^r)}.
\end{equation}
Here, $X_g^s=\{\textbf{x}_i|x^s_i=o^s_g, i\in N_X\}$ is a subset of $X$ with the $s$th values of all its objects equal to ${o}^s_g$, and the function $\text{card}(\cdot)$ counts the cardinality of a set. Then, we define the distance between two possible values (e.g., ${o}^r_m$ and ${o}^r_h$ of attribute $A^r$) with respect to another attribute (e.g. $A^s$) as follows:
\begin{equation}\label{eq:measure}
\text{dist}^{rs}( {o}^r_m, {o}^r_h)=\left\{
\begin{array}{ll}
\psi(\textbf{u}^{rs}_m,\textbf{u}^{rs}_h), &  A^s\in A^{(\text{ord})}\\
\dfrac{1}{v^s}\sum_{g=1}^{v^s}\psi(\textbf{u}^{rg}_{m},\textbf{u}^{rg}_{h}), &  A^s\in A^{(\text{nom})},\\
\end{array}
\right.
\end{equation}
where $\psi(\cdot,\cdot)$ computes the distance between two probability distributions. For the nominal case (i.e., $A^s\in A^{(\text{nom})}$), the distance with respect to $A^s$ is computed as the mean of the distances with respect to the ordinal attributes $A^g\in B^s$ that are converted from a nominal attribute $A^s$ as shown in Fig.~\ref{fig:convert}. See Eq.~(\ref{eq:expand}) and corresponding discussions in Section~\ref{subsct:homo} for more details. As both $A^s$ in the ordinal case and $A^g$ in the nominal case are ordinal attributes, we only need to discuss how to define $\psi(\cdot,\cdot)$ in the ordinal case.

In the literature, although the distance between two probability distributions is commonly computed in the form of $l_1$-norm (i.e., $||\textbf{u}^{rs}_m-\textbf{u}^{rs}_h||_1$) or $l_2$-norm (i.e., $||\textbf{u}^{rs}_m-\textbf{u}^{rs}_h||_2$), they are not suitable here because they cannot preserve order relationship among possible values of an ordinal attribute. For example, given $\textbf{u}_1=[1,0,0,0]^\top$, $\textbf{u}_2=[0,1,0,0]^\top$, $\textbf{u}_3=[0,0,0,1]^\top$, we have $||\textbf{u}_1-\textbf{u}_2||_1=||\textbf{u}_1-\textbf{u}_3||_1$. However, if $\textbf{u}_1$, $\textbf{u}_2$, $\textbf{u}_3$ are obtained from an ordinal attribute, it is obvious that $\textbf{u}_1$ and $\textbf{u}_2$ are more similar than $\textbf{u}_1$ and $\textbf{u}_3$, because the two possible values that rank 1st and 2nd are more similar than the two possible values that rank 1st and 4th. To preserve the order relationship, we define $\psi(\cdot,\cdot)$ as the cost of transforming a probability distribution into another according to the structure of ordinal attribute shown in Fig.~\ref{fig:diff}, and  $\psi(\textbf{u}^{rs}_m,\textbf{u}^{rs}_h)$ can be written as
\begin{equation}\label{eq:signature_dist}
\psi(\textbf{u}^{rs}_m,\textbf{u}^{rs}_h)=\frac{\sum_{t=1}^{v^s-1}|\sum_{g=1}^{t}\left(p( {o}_g^s| {o}_m^r)-p( {o}_g^s| {o}_h^r)\right)|}{v^s-1}.
\end{equation}
The distance defined in Eq.~(\ref{eq:signature_dist}) computes the minimum moving cost for transforming $\textbf{u}^{rs}_m$ into $\textbf{u}^{rs}_h$ (or $\textbf{u}^{rs}_h$ into $\textbf{u}^{rs}_m$), where $|\sum_{g=1}^{t}\left(p( {o}_g^s|o_m^r)-p(o_g^s|o_h^r)\right)|$ in Eq.~(\ref{eq:signature_dist}) is the total `supplies' or `demands' at location $o_t^s$ that should be moved to locations ${o}_{t+1}^s,{o}_{t+2}^s,...,{o}_{v^s}^s$ for offsetting. During the above computation, the `moving distance' between adjacent values is 1 because the prior knowledge we have is that the rank of a possible value is different from its adjacent possible value(s) by 1. A toy example shown in Fig.~\ref{fig:toy_example} intuitively illustrates the computation process.
\begin{figure}[t]
\centering
\includegraphics[width=3.5in]{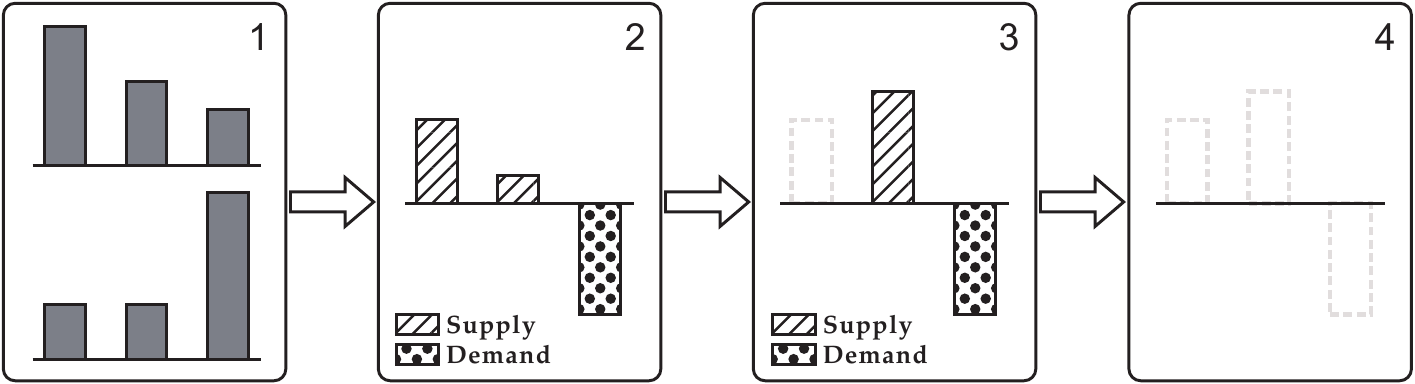}
\caption{Computation process of Eq.~(\ref{eq:signature_dist}). In step 1, we have $\textbf{u}^{rs}_m=[0.5,0.3,0.2]^\top$, i.e. the upper histogram, and $\textbf{u}^{rs}_h=[0.2,0.2,0.6]^\top$, i.e. the lower histogram. To transform $\textbf{u}^{rs}_m$ into $\textbf{u}^{rs}_h$, we first subtract them and obtain the histogram $[0.3,0.1,-0.4]^\top$ in step 2. The slash-filled bins indicate supplies, and the dot-filled bin indicates demand. Then, 0.3 supply at the first place is moved to the second place with 0.1 supply, the moving cost is (0.3$\times$1)/2 = 0.15. In step 3, the total 0.4 supply at the second place is moved to the third place with 0.4 demand, the moving cost is (0.4$\times$1)/2 = 0.2. Since the supply and demand exactly offset each other after step 3, the transforming is completed in step 4, and the total transforming cost is 0.15 + 0.2 = 0.35.}
\label{fig:toy_example}
\end{figure}

Eq.~(\ref{eq:signature_dist}) elaborately reflects the distance between two probability distributions obtained from an ordinal attribute, and we discuss it in detail below:
\begin{itemize}
\item According to our design, `supplies' and `demands' are moved strictly according to the structure of ordinal attribute as shown in Fig.~\ref{fig:diff}. It turns out that the order relationship among the possible values is taken into account in computing the distance between two distributions by Eq.~(\ref{eq:signature_dist}).
\item It is intuitive that two more different distributions yield more `supplies' and `demands' for moving, and thus result in a larger distance computed by Eq.~(\ref{eq:signature_dist}), which is consistent with the general distance definitions like Manhattan and Euclidean distance.
\item In terms of the form, Eq.~(\ref{eq:signature_dist}) can be viewed as a special case of the Earth Movers' Distance (EMD) \cite{emd1,emd2,emd3}, as Eq.~(\ref{eq:signature_dist}) only permits `moving' between adjacent bins of histograms. However, Eq.~(\ref{eq:signature_dist}) is designed under the guidance of the proposed graph structure shown in Fig.~\ref{fig:diff}, which is very different from the motivation and principle for designing EMD.
\end{itemize}

Eq.~(\ref{eq:measure}) and (\ref{eq:signature_dist}) have defined the distance between two possible values with respect to an attribute. Then, according to the structures of nominal and ordinal attributes studied in Section~\ref{subsct:structure}, we define the overall distance between two possible values by combining their distances with respect to each attribute as follows:
\begin{equation}\label{eq:dist_basic2}
\text{dist}^r( {o}^r_m, {o}^r_h)=\left\{
\begin{array}{l}
\dfrac{1}{d}\sum_{s=1}^d\sum_{t=\min(m,h)}^{\max(m,h)-1}\text{dist}^{rs}( {o}^r_t, {o}^r_{t+1})\\
\qquad\qquad\qquad\qquad\qquad\qquad A^r\in A^{(\text{ord})}\\
\dfrac{1}{d}\sum_{s=1}^d\text{dist}^{rs}( {o}^r_m, {o}^r_h),\quad\ \ \  A^r\in A^{(\text{nom})}.\\
\end{array}
\right.
\end{equation}

Based on Eq.~(\ref{eq:dist_basic2}), the distance between two data objects $\textbf{x}_i$ and $\textbf{x}_j$ with their $r$th values denoted as $x_i^r=o^r_m$ and $x^r_j=o^r_h$, respectively, can be written as
\begin{equation}\label{eq:dist_object}
\text{dist}(\textbf{x}_i,\textbf{x}_j)=\frac{1}{d}\sum_{r=1}^d\text{dist}^r(o^r_m,o^r_h).
\end{equation}

\begin{theorem}\label{the:metric}
Distance measure defined in Eq.~(\ref{eq:measure})-(\ref{eq:dist_object}) is a distance metric.
\end{theorem}
\begin{prf}
According to Eq.~(\ref{eq:measure})-(\ref{eq:dist_basic2}), it is clear that the defined intra-attribute distance satisfies the following properties for any $m,h,t\in N^r_O$ and $r\in N_A$:
\begin{enumerate}
\item $\text{dist}^r( {o}^r_m, {o}^r_h)\geq0$;
\item $ {o}^r_m= {o}^r_h\Leftrightarrow\text{dist}^r( {o}^r_m, {o}^r_h)=0 $;
\item $\text{dist}^r( {o}^r_m, {o}^r_h)=\text{dist}^r( {o}^r_h, {o}^r_m)$;
\item $\text{dist}^r( {o}^r_m, {o}^r_h)\leq\text{dist}^r( {o}^r_m, {o}^r_t)+\text{dist}^r( {o}^r_t, {o}^r_h)$.
\end{enumerate}
According to Eq.~(\ref{eq:dist_object}), it is clear that the following properties hold for any $i,j,l\in N_X$:
\begin{enumerate}
\item $\text{dist}(\textbf{x}_i,\textbf{x}_j)\geq0$;
\item $\textbf{x}_i=\textbf{x}_j\Leftrightarrow\text{dist}(\textbf{x}_i,\textbf{x}_j)=0$;
\item $\text{dist}(\textbf{x}_i,\textbf{x}_j)=\text{dist}(\textbf{x}_j,\textbf{x}_i)$;
\item $\text{dist}(\textbf{x}_i,\textbf{x}_j)\leq \text{dist}(\textbf{x}_i,\textbf{x}_l)+\text{dist}(\textbf{x}_l,\textbf{x}_j)$.
\end{enumerate}
The defined distance measure satisfies all the distance metric properties. \qed
\end{prf}

In practice, a set of distance matrices, i.e. $D=\{\textbf{D}^r|r\in N_A\}$ where $\textbf{D}^r$ is a $v^r\times v^r$ symmetric matrix storing intra-attribute distances of $A^r$, can be computed before clustering. The $(m,h)$th entry of $\textbf{D}^r$ is denoted as $d^r_{mh}$ where $d^r_{mh}=\text{dist}^r(o^r_m,o^r_h)$. With $D$, distances can be directly read off during clustering.

\begin{theorem}\label{the:dist_complexity}
Time complexity for computing the distance matrices $D$ is $O(nd^2+d^2V^3)$.
\end{theorem}
\begin{prf}
Conditional probability distribution $\textbf{u}_m^{rs}$ with $r,s\in N_A$ and $m\in N^r_O$ should be obtained before distance computation. For each $\textbf{u}_m^{rs}$, $o^r_m$'s corresponding values on $A^s$ should be scanned once with time complexity $O(\text{card}(X^r_m))$, and for all the $\textbf{u}_m^{rs}$ with $m\in N^r_O$, the scan is with time complexity $O(n)$. Such scan should be performed for each pair of attributes, and thus the time complexity is $O(nd^2)$.

Given a pair of possible values ${o}^r_m$ and ${o}^r_h$ with $m,h\in N^r_O$ and $r\in N_A$, the time complexity for computing the distance between them based on the known $\textbf{u}_m^{rs}$ and $\textbf{u}_h^{rs}$ is $O(V)$ in both the two cases of Eq.~(\ref{eq:measure}). Note that $V=\max(v^1,v^2,...,v^d)$ is the maximum number of possible values among all the attributes, which is adopted to simplify the time complexity analysis. To obtain $\textbf{D}^r$, the time complexity for computing the $V(V-1)/2$ intra-attribute distances is $O(dV^3)$ in the case $A^r\in A^{(\text{nom})}$ of Eq.~(\ref{eq:dist_basic2}). In the case $A^r\in A^{(\text{ord})}$ of Eq.~(\ref{eq:dist_basic2}), distance between possible values with order difference 1 can be computed first, and then the distance between possible values with order difference 2, 3, ..., $V-1$ can be successively computed based on the distances computed in the previous step. Therefore, the time complexity for computing the $V(V-1)/2$ intra-attribute distances is also $O(dV^3)$ in the case $A^r\in A^{(\text{ord})}$ of Eq.~(\ref{eq:dist_basic2}). The time complexity for computing a total of $d$ distance matrices $\textbf{D}^r$ is thus $O(d^2V^3)$. Hence, the overall time complexity for obtaining $D$ is $O(nd^2+d^2V^3)$.
\qed
\end{prf}

\section{Clustering Based on Intra-Attribute \\ Distance Weighting}\label{sct:clustering}
Often, separately treating the cross-coupled distance defining and data clustering results in a suboptimal solution. This section will therefore propose a learning mechanism that adjusts the defined intra-attribute distances to suit certain clustering tasks. We have constructed graph-like structure for the intra-attribute possible values to define their distances in Section~\ref{sct:dist}, and will learn the weights of the distances in an iterative way with data clustering in the following subsections. Before introducing the details of our algorithm, let us conceptually discuss the existing methods whose learning paradigms are intuitively similar to ours.

Several clustering of bandits algorithms \cite{r1bandit5,r1bandit1,r1bandit3} have been proposed to construct graph for the objects (i.e., users in their application scenarios) and dynamically perform graph clustering according to the item preference of users over time. Further, the one in \cite{r1bandit2} captures the collaborative effects of the users, and the one in \cite{r1bandit4} captures the bi-collaborative effects between users and items by iteratively partitioning user and item graphs. The commonality of our method and the above-mentioned ones is that they all iteratively learn (1) the relationship between objects and (2) certain predictions for the objects. The differences are: (1) we construct graphs only for studying the distance definition between possible values, while most of the above-mentioned methods construct graphs for objects and cutting the graphs for object partitioning, (2) we optimize the prediction from objects to object clusters, while they optimize the prediction of items to be recommended to users. Although the above-mentioned methods are not solving the same type of problem as ours, their paradigms can provide inspiration for applying our method in more complex environments in the future, for example, in on-line or distributed situations.
In the following, we will elaborate how to solve the two problems $\textbf{P. 1}$ and $\textbf{P. 2}$ stated in Section~\ref{sct:preliminary}, and present the clustering algorithm, together with the time-complexity analysis.

\subsection{Update $\textbf{Q}$ As Given $W$ and $P$}\label{subsct:partition}
The process of solving problem $\textbf{P. 1}$ is to obtain a data partition according to a certain distance measure and cluster representation, which actually adopts the same basic idea as most $k$-modes-type algorithms. The difference is that we use the statistical information $P$ (defined in Section~\ref{sct:preliminary}) instead of cluster modes in representing the clusters, which ensures the extraction of more rich information for learning distance weights $W$ in solving problem $\textbf{P. 2}$. Specifically, the details of solving $\textbf{P. 1}$ are presented as follows. According to the objective function defined in Eq.~(\ref{eq:objective}), $\textbf{P. 1}$ is solved by fixing $W=\hat{W}$ and $P=\hat{P}$, and computing $\textbf{Q}$. Given distance matrices $D$, $\textbf{Q}$ is computed by
\begin{equation}\label{eq:update_q}
q_{il}=\left\{
\begin{array}{ll}
1, & \text{if}\ {l=\arg\min_y\text{dist}(\textbf{x}_i,C_y)}\\
   & \quad\ \ =\arg\min_y\sum_{r=1}^d\sum_{h=1}^{v^r}\hat{w^r_{mh}}d^r_{mh}\hat{p^r_{yh}}\\
0, & \text{otherwise}
\end{array}
\right.\\
\end{equation}
for $\textbf{x}_i$ with $x^r_i=o^r_m$. Since we represent the clusters using their probability distributions by $P$ instead of using cluster modes, the form of the solution in Eq.~(\ref{eq:update_q}) is different form the conventional $k$-modes-type algorithms. Thus, solution to $\textbf{P. 1}$ is also rigorously given in Theorem~\ref{the:solve_p1}.
\begin{theorem}\label{the:solve_p1}
Let $W$ and $P$ be fixed, $Z(\textbf{Q},\hat{P},\hat{W})$ is minimized iff $\textbf{Q}$ is computed utilizing Eq.~(\ref{eq:update_q}).
\end{theorem}
\begin{prf}
For any given $W=\hat{W}$ and $P=\hat{P}$, all the inner sums of the quantity
\begin{flalign}\nonumber
&Z(\textbf{Q},\hat{P},\hat{W})=\sum_{i=1}^n\sum_{l=1}^kq_{il}\text{dist}(\textbf{x}_i,C_l)&
\end{flalign}
are nonnegative and independent. Let $o^r_m=x^r_i$, we can write the inner sum contributed by $\textbf{x}_i$ as
\begin{flalign}\nonumber
&z_i=\sum_{l=1}^kq_{il}\text{dist}(\textbf{x}_i,C_l)=\sum_{l=1}^kq_{il}\sum_{r=1}^d\sum_{h=1}^{v^r}\hat{w^r_{mh}}d^r_{mh}\hat{p^r_{lh}}.&\nonumber
\end{flalign}
Let $z_{il}=\sum_{r=1}^d\sum_{h=1}^{v^r}\hat{w^r_{mh}}d^r_{mh}\hat{p^r_{lh}}$, which is the inner sum contributed by $\textbf{x}_i$ in $C_l$. We then obtain
\begin{flalign}\nonumber
&z_i=\sum_{l=1}^kq_{il}z_{il}.&
\end{flalign}
Since $\sum_{l=1}^kq_{il}=1$ and $q_{il}\in\{0,1\}$, it is clear that $z_i$ is minimized iff the minimum $z_{il}$ is assigned with $q_{il}=1$ where $l$ is determined by
\begin{flalign}\nonumber
l&=\arg\min_yz_{iy}=\arg\min_y\sum_{r=1}^d\sum_{h=1}^{v^r}\hat{w^r_{mh}}d^r_{mh}\hat{p^r_{yh}}&\nonumber
\end{flalign}
and the other $z_{il}$s are assigned with $q_{il}=0$. The result follows.\qed
\end{prf}

We have presented the solution of updating \textbf{Q}. Each time a new \textbf{Q} is obtained, the cluster representation $P$ is updated accordingly, and such process is iterated until convergence.

\subsection{Update $W$ As Given $\textbf{Q}$ and $P$}\label{subsct:weight}
In Section~\ref{subsct:partition}, we have presented the solution of $\textbf{P. 1}$. Then, $\textbf{P. 2}$ should be solved based on the present \textbf{Q} and $P$ to learn distance weights $W$. In this part, a novel learning scheme is designed by mining the latent interaction between data partition and intra-attribute distances, so as to seek for more appropriate data partition in the next iteration based on the defined intra-attribute distances and the newly learned $W$. The details of solving $\textbf{P. 2}$ are presented as follows. Given fixed $\hat{\textbf{Q}}$ and $\hat{P}$, the objective function defined by Eq.~(\ref{eq:objective}) can be written as
\begin{align}\nonumber\label{eq:obj_cate}
Z(\hat{\textbf{Q}},\hat{P},W)&=\sum_{i=1}^n\sum_{l=1}^k\hat{q_{il}}\text{dist}(\textbf{x}_i,C_l)\\\nonumber
 &=\sum_{i=1}^n\sum_{l=1}^k\hat{q_{il}}\sum_{r=1}^d\sum_{h=1}^{v^r}w^r_{mh}d^r_{mh}\hat{p^r_{lh}}\\
 &=\sum_{r=1}^d\sum_{m=1}^{v^r}\sum_{h=1}^{v^r}w^r_{mh}d^r_{mh}\sum_{l=1}^k\frac{f^r_{lm}f^r_{lh}}{f_l}
\end{align}
where $f^r_{lm}=\text{card}(X^r_{m}\cap X_{C_l})$ and $f^r_{lh}=\text{card}(X^r_{h}\cap X_{C_l})$ are the total number of objects in $C_l$ with their $r$th values equal to $o^r_m$ and $o^r_h$, respectively, $f_l=\text{card}(X_{C_l})$ is the number of objects in $C_l$, and we have $f^r_{lh}/f_l=\hat{p^r_{lh}}$. In most $k$-modes-type algorithm with attribute weighting mechanism \cite{wkm,wkm2,mwkm}, Lagrangian multiplier is used to convert the constrained weights computation problem into an unconstrained problem so that the optimal attribute weights can be computed directly in each iteration. However, solving our intra-attribute distance weighting problem in this way may encounter two awkward issues:
\begin{itemize}
\item \textit{Frequency Effect:} For attribute weighting, each attribute has the identical number of values, which is the basis for success in making the computed weights comparable. However, the occurrence frequencies of intra-attribute distances (i.e. $\sum_{l=1}^kf^r_{lm}f^r_{lh}$ of $d^r_{mh}$) are usually different from each other, which makes the computed weights of intra-attribute distances incomparable.
\item \textit{Co-occurrence Sparsity:} It is common for a real categorical data set that an intra-attribute distance (i.e. $d^r_{mh}$) never occur in a cluster, so that no statistical information is provided for the computation of its corresponding weight $w^r_{mh}$. If we set such weights to 0, the problem still cannot be fixed because there are many such weights preventing the algorithm from convergence.
\end{itemize}

We propose a novel intra-attribute distance weight updating scheme to circumvent the above-discussed issues. In general, a larger $d_{mh}^r$ indicates that the two corresponding possible values $o^r_m$ and $o^r_h$ are more dissimilar. That is, $d_{mh}^r$ is expected to contribute more in partitioning the objects in $X^r_m=\{\textbf{x}_i|x^r_i=o^r_m, i\in N_X\}$ and the objects in $X^r_h=\{\textbf{x}_i|x^r_i=o^r_h, i\in N_X\}$ into different clusters. Thus, given a data partition $\hat{\textbf{Q}}$, if $d_{mh}^r$ is larger but more objects in $X^r_m$ and $X^r_h$ are assigned into the same cluster, it is indicated that $d_{mh}^r$ does not contribute in partitioning the objects $X^r_m$ and $X^r_h$ into different clusters as expected. Accordingly, the weight of $d_{mh}^r$ should be estimated as its expectation in reducing $Z(\hat{\textbf{Q}},\hat{P},W)$, and we have
\begin{equation}\label{eq:propto}
w_{mh}^r\propto [d^r_{mh}\sum_{l=1}^k\frac{f^r_{lm}f^r_{lh}}{f^r_mf^r_h}]^{-1}
\end{equation}
where $f^r_m=\text{card}(X^r_m)=\sum_{l=1}^kf^r_{lm}$ and $f^r_h=\text{card}(X^r_h)=\sum_{l=1}^kf^r_{lh}$ are the intra-cluster occurrence frequencies of $o^r_m$ and $o^r_h$, respectively. The term $(f^r_{lm}f^r_{lh})/(f^r_mf^r_h)$ quantifies the occurrence of $d^r_{mh}$ in $C_l$ as a joint occurrence probability of $o^r_m$ and $o^r_h$ in $C_l$, which avoids the \textit{Frequency Effect}. If we directly update $W$ by $w_{mh}^r=[d^r_{mh}\sum_{l=1}^k(f^r_{lm}f^r_{lh})/(f^r_mf^r_h)]^{-1}$ according to Eq.~(\ref{eq:propto}), the \textit{Co-occurrence Sparsity} issue may make the values of different weights vary greatly in the interval $[1/d^r_{mh},\infty)$, which may cause non-convergence. Thus, we discuss how to novelly circumvent the \textit{Co-occurrence Sparsity} issue in the following.

\begin{lemma}\label{lem:constant}
Given an arbitrary partition $\textbf{Q}$ of data set $S$, sum of the intra-cluster distance $E^{r(\text{intra})}_{mh}$
and inter-cluster distance $E^{r(\text{inter})}_{mh}$ contributed by $d^r_{mh}$ is a constant.
\end{lemma}
\begin{prf}
We first note that $E^{r(\text{intra})}_{mh}=d^r_{mh}\sum_{l=1}^kf^r_{lm}f^r_{lh}$ and $E_{mh}^{r(\text{inter})}=d^r_{mh}\sum_{s=1}^{k-1}\sum_{u=s+1}^k(f^r_{sm}f^r_{uh}+f^r_{um}f^r_{sh})$.
Let $E^{r(\text{total})}_{mh}=E_{mh}^{r(\text{intra})}+E_{mh}^{r(\text{inter})}$, we have
\begin{flalign}\nonumber
E^{r(\text{total})}_{mh}&=d^r_{mh}\left(\sum_{l=1}^kf^r_{lm}f^r_{lh}+\sum_{s=1}^{k-1}\sum_{u=s+1}^k\left(f^r_{sm}f^r_{uh}+f^r_{um}f^r_{sh}\right)\right)&\\\nonumber
&=d^r_{mh}\left(\sum_{s=1}^k\sum_{u=s}f^r_{sm}f^r_{uh}\notag\right.&\\\nonumber
&\ \ \ \qquad\notag\left.+\sum_{s=1}^{k-1}\sum_{u=s+1}^kf^r_{sm}f^r_{uh}+\sum_{s=1}^{k-1}\sum_{u=s+1}^kf^r_{um}f^r_{sh}\right)&\\\nonumber
&=d^r_{mh}\sum_{s=1}^k\sum_{u=1}^kf^r_{sm}f^r_{uh}=d^r_{mh}\sum_{s=1}^kf^r_{sm}\sum_{u=1}^kf^r_{uh}\\\nonumber
&=d^r_{mh}f^r_mf^r_h.&
\end{flalign}
Since $d^r_{mh}$, $f^r_m$, and $f^r_h$ are constants for a given data set $S$, it is clear that $E^{r(\text{total})}_{mh}$ is a constant. The result follows. \qed
\end{prf}

\noindent Based on Lemma~\ref{lem:constant}, Eq.~(\ref{eq:propto}) can be transformed to avoid the \textit{Co-occurrence Sparsity} issue.

\begin{lemma}\label{lem:propto}
Given Eq.~(\ref{eq:propto}), $w_{mh}^{r}\propto d_{mh}^r\sum_{s=1}^{k-1}\sum_{u=s+1}^k(f^r_{sm}f^r_{uh}+f^r_{um}f^r_{sh})/(f^r_mf^r_h)$ holds when $\exists\ l\in N_C$ so that $f^r_{lm}f^r_{lh}\neq0$.
\end{lemma}
\begin{prf}
Let $H^r_{mh}=d_{mh}^r\sum_{s=1}^{k-1}\sum_{u=s+1}^k(f^r_{sm}f^r_{uh}+f^r_{um}f^r_{sh})/(f^r_mf^r_h)$. We first prove that $H^r_{mh}<E^{r(\text{total})}_{mh}/(f^r_mf^r_h)$. According to the proof of Lemma~\ref{lem:constant}, we derive
\begin{flalign}\label{eq:the_h}
&\frac{E^{r(\text{total})}_{mh}}{f^r_mf^r_h}-H^r_{mh}=\frac{d^r_{mh}\sum_{l=1}^kf^r_{lm}f^r_{lh}}{f^r_mf^r_h}.&
\end{flalign}
Since $\exists\ l\in N_C$ so that $f^r_{lm}f^r_{lh}\neq0$, and $f^r_{lm}$ and $f^r_{lh}$ are non-negative integers, we have $\sum_{l=1}^kf^r_{lm}f^r_{lh}>0$; Since $f^r_m$ and $f^r_h$ are positive constants and $d^r_{mh}>0$ for two different possible values $o^r_m$ and $o^r_{h}$, we then have
\begin{flalign}\nonumber
&\dfrac{d^r_{mh}\sum_{l=1}^kf^r_{lm}f^r_{lh}}{f^r_mf^r_h}>0\ \Rightarrow\ H^r_{mh}<\frac{E^{r(\text{total})}_{mh}}{f^r_mf^r_h}.&
\end{flalign}
From Eq.~(\ref{eq:propto}) and (\ref{eq:the_h}), we derive
\begin{flalign}\nonumber
&w^{r}_{mh}\propto\frac{1}{\frac{E^{r(\text{total})}_{mh}}{f^r_mf^r_h}-H^r_{mh}}.&
\end{flalign}
Since we have proved $H^r_{mh}<E^{r(\text{total})}_{mh}/(f^r_mf^r_h)$, and $E^{r(\text{total})}_{mh}/(f^r_mf^r_h)$ is a constant, it is clear that the value of $w^r_{mh}$ is proportional to the value of $H^r_{mh}$, which can be written as
\begin{flalign}\label{eq:propto2}
w_{mh}^{r}\propto\frac{d_{mh}^r\sum_{s=1}^{k-1}\sum_{u=s+1}^k(f^r_{sm}f^r_{uh}+f^r_{um}f^r_{sh})}{f^r_mf^r_h}.
\end{flalign}
The result follows. \qed
\end{prf}

According to Lemma~\ref{lem:propto}, the weights of intra-attribute distances are updated by
\begin{equation}\label{eq:novel_update}
w_{mh}^{r\text{(new)}}=\dfrac{d_{mh}^r\sum_{s=1}^{k-1}\sum_{u=s+1}^k(f^r_{sm}f^r_{uh}+f^r_{um}f^r_{sh})}{f^r_mf^r_h}.
\end{equation}
Eq.~(\ref{eq:novel_update}) is obtained with restriction $\exists\ l\in N_C$ so that $f^r_{lm}f^r_{lh}\neq0$. We also demonstrate that when $\forall\ l\in N_C$, $f^r_{lm},f^r_{lh}=0$, Eq.~(\ref{eq:novel_update}) is still meaningful. $\forall\ l\in N_C$,\ $f^r_{lm},f^r_{lh}=0$ indicates that the objects in $X_m^r$ and the objects in $X_h^r$ never appear in the same cluster, which means that the contribution of $d^r_{mh}$ in partitioning the objects in $X_m^r$ and the objects in $X_h^r$ into different clusters reaches the maximum, i.e. $\sum_{s=1}^{k-1}\sum_{u=s+1}^k(f^r_{sm}f^r_{uh}+f^r_{um}f^r_{sh})/(f^r_mf^r_h)=1$. Since the value of $\sum_{s=1}^{k-1}\sum_{u=s+1}^k(f^r_{sm}f^r_{uh}+f^r_{um}f^r_{sh})/f^r_mf^r_h$ is in the interval $[0,1]$, it is clear that weights updating utilizing Eq.~(\ref{eq:novel_update}) will not be influenced by the \textit{Co-occurrence Sparsity} issue. We also use soft-max, i.e. $w_{mh}^{r}=w_{mh}^{r\text{(new)}}/\sum_{s=1}^{d}\sum_{g=1}^{v^s-1}\sum_{t=g+1}^{v^s}w_{gt}^{s\text{(new)}}$, to make the updated weights satisfy $\sum_{r=1}^d\sum_{m=1}^{v^r-1}\sum_{h=m+1}^{v^r}w^r_{mh}=1$.

Advantages of the proposed weights updating scheme are summarized below:
\begin{itemize}
\item \textit{Frequency Dominance} issue is avoided.
\item \textit{Co-occurrence Sparsity} issue is novelly circumvented.
\item It is parameter-free, and the clustering algorithm based on it (see Section~\ref{subsct:algorithm}) always converge quickly, which has been illustrated in Section~\ref{sct:experiments}.
\end{itemize}

\subsection{Complete Clustering Algorithm}\label{subsct:algorithm}

The complete clustering algorithm called HD-NDW integrates the solutions of $\textbf{P. 1}$ and the Novel Distance Weighting (NDW) mechanism for solving $\textbf{P. 2}$. As described in Algorithm~\ref{alg:algorithm}, it iteratively updates the data partitions and weights of the distances defined by the Homogeneous Distance (HD) metric for data partitioning. More specifically, $\textbf{Step 1}$ acts as a complete clustering algorithm that learns a data partition, which provides information for updating the weights of distances in $\textbf{Step 2}$. This is why we put $\textbf{Step 1}$ before $\textbf{Step 2}$ in HD-NDW. After reasonable distance weights are learned according to the data partition, the weights are fed back to $\textbf{Step 1}$ for learning more appropriate data partition, and such procedures iterate until convergence. Time complexity of HD-NDW is analyzed in Theorem~\ref{the:algorithm_complexity}.

\begin{algorithm}[t]
\caption{HD-NDW Clustering Algorithm}
\label{alg:algorithm}
\setlength{\parskip}{-1.3em}
\ \\
\setlength{\parskip}{1em}
\begin{flushleft}
\setlength{\parskip}{-1em}
\hangafter=1
\setlength{\hangindent}{1em}
\textbf{Input:} Data set $S$, number $k$ of clusters, distance matrices $D$.\\

\setlength{\parskip}{-1em}
\hangafter=1
\setlength{\hangindent}{1em}
\textbf{Output:} Partition $\textbf{Q}$.\\

\setlength{\parskip}{-1em}
\hangafter=1
\setlength{\hangindent}{1em}
\textbf{Step 0:} Initialize the time-step by $\tau=0$; Initialize $P^{(\tau)}$ and $W^{(\tau)}$;\\

\setlength{\parskip}{-1em}
\hangafter=1
\setlength{\hangindent}{1em}
\textbf{Step 1:} Fix $W^{(\tau)}$ and $P^{(\tau)}$, iteratively update $\textbf{Q}^{(\tau)}$ by Eq.~(\ref{eq:update_q}) and update $P^{(\tau)}$ according to $\textbf{Q}^{(\tau)}$ until convergence, obtain $\textbf{Q}^{(\tau+1)}$ and $P^{(\tau+1)}$; If $\textbf{Q}^{(\tau+1)}\neq\textbf{Q}^{(\tau)}$, go to \textbf{Step 2}; Otherwise, stop and \textbf{Output} $\textbf{Q}^{(\tau)}$.\\

\setlength{\parskip}{-1em}
\hangafter=1
\setlength{\hangindent}{1em}
\textbf{Step 2:} Fix $\textbf{Q}^{(\tau+1)}$ and $P^{(\tau+1)}$, update $W^{(\tau)}$ by Eq.~(\ref{eq:novel_update}), obtain $W^{(\tau+1)}$; Update the time-step by $\tau=\tau+1$, go to \textbf{Step 1};
\setlength{\parskip}{-2.5em}
\end{flushleft}
\end{algorithm}

\begin{theorem}\label{the:algorithm_complexity}
Time complexity of Algorithm~\ref{alg:algorithm} is $O(E(kdVnI+dn+dV^2k))$, supposing $\textbf{Step 1}$ needs $I$ iterations to converge, and the loop of $\textbf{Step 1}$ and $\textbf{2}$ needs $E$ iterations to converge.
\end{theorem}
\begin{prf}
In $\textbf{Step 1}$, time complexity for computing the values of a row of $\textbf{Q}^{(\tau)}$ is $O(kdV)$ because there are $k$ clusters to be considered, and for each cluster, the distance is computed based on the $d$ intra-attribute distances stored in $D$, and each attribute has a maximum of $V$ possible values. See the proof of Theorem~\ref{the:solve_p1} for more details of the computing of $\textbf{Q}^{(\tau)}$. Since there are $n$ rows in $\textbf{Q}^{(\tau)}$ and $\textbf{Step 1}$ repeats $I$ times, the total time complexity of $\textbf{Step 1}$ is $O(kdVnI)$.

According to the proof of Lemma~\ref{lem:constant}, the term $\sum_{s=1}^{k-1}\sum_{u=s+1}^k(f^r_{sm}f^r_{uh}+f^r_{um}f^r_{sh})/(f^r_mf^r_h)$ in Eq.~(\ref{eq:novel_update}) can be directly computed by $1-\sum_{l=1}^kf^r_{lm}f^r_{lh}/(f^r_mf^r_h)$. Before the computation, we should first obtain the set of occurrence frequency matrices $F=\{\textbf{F}^1,\textbf{F}^2,...,\textbf{F}^d\}$ where $\textbf{F}^r$ is a $k\times v^r$ matrix storing the occurrence frequencies of $A^r$'s possible values in each cluster, and the $(l,m)$th entry of $\textbf{F}^r$ is $f^r_{lm}$. To obtain $F$, the $d$ values of each data object $\textbf{x}_i$ should be scanned once according to the corresponding $q_{il}=1$ in $\textbf{Q}$. Since there are $n$ objects in total, the time complexity for obtaining $F$ is $O(dn)$. It is therefore clear that the time complexity for computing the $dV(V-1)/2$ weights according to each of the $k$ clusters using Eq.~(\ref{eq:novel_update}) is $O(dn+dV^2k)$ in $\textbf{Step 2}$.

Since the loop of $\textbf{Step 1}$ and $\textbf{2}$ repeats $E$ times, the time complexity of HD-NDW is $O(E(kdVnI+dn+dV^2k))$.
\qed
\end{prf}

\section{Experiments}\label{sct:experiments}

We conduct a series of experiments on various benchmark and real data sets to evaluate the proposed clustering method. We first describe the experimental settings. Then, we demonstrate and discuss the experimental results.

\subsection{Experimental Settings}\label{subsct:ex_settings}

\subsubsection{Experimental Design}\label{subsubsct:ex_design}

Five experiments are designed as follows:
\begin{itemize}
\item \textit{Clustering Performance of HD-NDW.} We compare HD-NDW with various clustering algorithms on mixed, ordinal, and nominal categorical data sets to illustrate the superiority of HD-NDW.
\item \textit{Effectiveness of HD.} HD is a core component of HD-NDW. We compare HD and various distance measures by combining them with the simplest $k$-modes clustering algorithm to illustrate the effectiveness of HD.
\item \textit{Effectiveness of NDW.} NDW is also a core component of HD-NDW. We compare HD-NDW and its non-weighting version to prove the effectiveness NDW.
\item \textit{Convergence Evaluation.} Convergence curves of HD-NDW on various data sets are demonstrated to illustrate its effectiveness and fast convergence.
\item \textit{Computational Efficiency Evaluation.} We compare the execution time of various clustering methods on synthetic data sets to illustrate the efficiency of HD-NDW.
\end{itemize}
For all the experiments, the number $k$ of the clusters is set at the true number $k^*$ of the clusters according to the data label. We run all the experiments 50 times and report the average results.

\subsubsection{Validity Indices}\label{subsubsct:ex_indices}

We select the commonly used Adjusted Rand Index (ARI) \cite{ex2} because it is powerful in discriminating clustering performance \cite{ex3,ex4}. Normalized Mutual Information (NMI) \cite{ex5}\cite{jdm} is selected to evaluate clustering performance from the perspective of information theory \cite{intro9}. To make the evaluation comprehensive, the traditional Clustering Accuracy (CA) \cite{ex0,ex1} is also selected. NMI and CA are in the interval $[0,1]$ and ARI is in the interval $[-1,1]$. For all these selected validity indices, a higher value indicates a better clustering performance. We also adopt Wilcoxon signed-rank test and Bonferroni-Dunn test \cite{r1test} to evaluate the statistical significance of the difference between clustering performance of different methods. In addition, we compute the averaged Intra- and Inter-Cluster Distance (ICD for short) \cite{adm} to intuitively demonstrate the cluster discrimination ability of different methods.

\subsubsection{Counterpart Selection}\label{subsubsct:ex_counterpart}

The most representative partitional clustering algorithms are selected as counterparts for the experiments. We select $k$-modes (KMD) \cite{kmd} because it is the most conventional one. We select Entropy-based Categorical data Clustering (ECC) \cite{ecc} because it is conventional and representative among the entropy-based clustering algorithms. We also select attribute Weighting $k$-modes (WKM) \cite{wkm}, Mixed-attribute Weighting $k$-modes (MWKM) \cite{mwkm}, and attribute Weighting and Object-cluster-similarity-based Clustering (WOC) \cite{woc} algorithms as another three counterparts. WKM and MWKM are two representative algorithms in the attribute-weighting clustering stream, and WOC is the most state-of-the-art one that extends the attribute weighting into subspace. Space structure-Based Clustering (SBC) \cite{r1compare1}, Coupled Data Embedding-based clustering (CDE) \cite{r1compare3}, and UNsupervised heTerogeneous couplIng lEarning-based clustering (UNTIE) \cite{r1compare4} are also chosen as the counterparts in the stream of data representation-based clustering. SBC has two versions, denoted as SBC-1 and SBC-2, whose difference is to adopt the different distance functions only. For simplicity, we just therefore report the performance of the one with better performance on each data set. Also, the state-of-the-art Distance Learning-based Clustering (DLC) \cite{dlc} algorithm is selected. Since it is designed for ordinal data only, we first perform the `simple coding' as discussed in Section~\ref{subsct:related_measure} to encode the nominal attributes of mixed data sets, and then perform DLC for clustering.

We select categorical data distance measures as counterparts of the proposed HD distance metric. We select Hamming distance metric \cite{hdm} because it is the most commonly used one in categorical data clustering. We also select Lin's Similarity Measure (LSM) \cite{lsm} as a representative for the stream of entropy-based measures, and Context-Based Distance Metric (CBDM) \cite{cbdmjournal} as a representative for the stream of context-based metrics. We also select three state-of-the-art categorical data distance metrics, i.e., Jia's Distance Metric (JDM) \cite{jdm}, Entropy-Based Distance Metric (EBDM) \cite{udm}, and Coupled Metric Similarity (CMS) \cite{r1compare2} as counterparts. We set the parameters of the above-mentioned counterparts at the values suggested by the corresponding papers.

\subsubsection{Data Sets}\label{subsubsct:ex_data}

We collect 15 data sets for the experiments, and the data statistics are shown in Table \ref{tb:statistics}.
\begin{table}[t]
\caption{Statistics of the 15 utilized data sets. ``\# Attribute'' of mixed categorical data sets indicates ``\# ordinal attributes + \# nominal attributes''}
\label{tb:statistics}
\centering
\begin{tabular}{l|l|r|r|r}
\toprule
Data type& Data set    & \# Instance  & \# Attribute & \# Class  \\
\midrule
\multirow{6}{*}{Mixed}
&Lenses   & 24       & 2+2          & 12        \\
&Assistant& 72       & 2+2          & 3         \\
&Hayes    & 132      & 2+2          & 3         \\
&Lym      & 148      & 3+15         & 4         \\
&Cancer   & 286      & 4+5          & 2         \\
&Nursery  & 12,960   & 7+1          & 4         \\
\midrule
\multirow{5}{*}{Ordinal}
&Photo    & 66       & 4          & 3         \\
&Selection& 488      & 4          & 9         \\
&Lecturer & 1,000    & 4          & 5         \\
&Social   & 1,000    & 10         & 4         \\
&Car      & 1,728    & 7          & 4         \\
\midrule
\multirow{4}{*}{Nominal}
&Soybean  & 47       & 21         & 4         \\
&Zoo      & 101      & 16         & 7         \\
&Solar    & 323      & 9          & 6         \\
&Voting   & 435      & 16         & 2         \\
\bottomrule
\end{tabular}
\end{table}
Among the six mixed categorical data sets (mixed data sets for short), Lenses, Breast Cancer (abbreviated as Cancer), Hayes-Roth (abbreviated as Hayes), Lymphography (abbreviated as Lym), and Nursery, are benchmark data sets collected from the UCI Machine Learning Repository (UCI-MLR)\footnote{http://archive.ics.uci.edu/ml/datasets.html\label{uci}} \cite{uci}, Assistant Evaluation (abbreviated as Assistant) is a real mixed categorical data set collected from university questionnaires. Among the five ordinal data sets, Lecturer Evaluation (abbreviated as Lecturer), Social Works (abbreviated as Social), and Employee Selection (abbreviated as Selection) are benchmark data sets collected from the Weka website\footnote{https://www.cs.waikato.ac.nz/ml/weka/datasets.html} \cite{weka}, Photo Evaluation (abbreviated as Photo) is a real ordinal data set collected from university questionnaires, and Car Evaluation (abbreviated as Car) is a benchmark data set collected from UCI-MLR. For all these five ordinal data sets, monotonic correlation exists among all the attributes, i.e., an object composed of higher ranked values always ranks higher in comparison with the other objects composed of lower ranked values \cite{ex6}. To utilize such known monotonicity, the original object-cluster distance is replaced with $\text{dist}(\textbf{x}_i,C_l)=|\text{dist}(\textbf{x}_i,\textbf{x}_0)-\text{dist}(\textbf{x}_i,\textbf{x}_0)|$ for the measures (i.e., LSM, EBDM, DLC, and HD) that are capable in distinguishing the order of values, in conducting the clustering experiment in Section~\ref{subsct:ex_clustering}. Note that $\textbf{x}_0$ here is a constructed object composed of the lowest ranked value of each attribute. All the four nominal data sets, i.e., Solar Flare (abbreviated as Solar), Zoo, Voting Records (abbreviated as Voting), and Soybean, are benchmark data sets collected from UCI-MLR.

\subsubsection{Initialization of HD-NDW}\label{subsubsct:ex_initialize}

In $\textbf{Step 0}$ of the proposed HD-NDW algorithm, values of $P$ and $W$ should be initialized. For $P$, although different initialization strategies can be utilized, we adopt a strategy similar to the random initialization of the conventional $k$-modes algorithm. That is, we randomly select $k^*$ objects as modes, and then assign values to the $k^*\times d$ vectors of $P$ accordingly. Taking the data set shown in Table~\ref{tb:example_mix_data} as an example, suppose we have $o^1_1=\uparrow$, $o^1_2=\uparrow\uparrow$, $o^1_3=\uparrow\uparrow\uparrow$, $o^1_4=\uparrow\uparrow\uparrow\uparrow$, $o^2_1=$non-special, $o^2_2=$vesicles, and $o^2_3=$chalices. If the 6th object in Table~\ref{tb:example_mix_data} (i.e. $\textbf{x}_6=[\uparrow\uparrow\uparrow\uparrow, \text{vesicles}]^\top$) is initialized as the mode of the 2nd cluster, then the corresponding two vectors in $P_2$ will be $\textbf{p}_2^1=[0,0,0,1]^\top$ and $\textbf{p}_2^2=[0,1,0]^\top$, respectively. For $W$, we uniformly initialize each weight of it to $1/(\sum_{r=1}^d v^r(v^r-1)/2)$. In this way, the sum of all the initialized weights equals to 1, which is equal to the sum of the weights after updating, as the updated weights will be processed using soft-max (see the discussions following Eq.~(\ref{eq:novel_update}) for more details). Another purpose of such a uniform initialization is to make the initialized weights have no effect on the learning of $\textbf{Step 1}$ in Algorithm~\ref{alg:algorithm}. If we randomly initialize $W$, inappropriate distance weights will prevent $\textbf{Step 1}$ from learning reasonable data partition, which will further influence the subsequent learning iterations.

\subsection{Clustering Performance Evaluation of HD-NDW}\label{subsct:ex_clustering}

\begin{table*}[t]
\caption{Clustering performance of various clustering algorithms on mixed and ordinal categorical data sets. The column of `$\Delta$' reports the improvements achieved by HD-NDW in comparison with the best-performing counterparts on different data sets. Results of significance tests are shown in Table~\ref{tb:sigalgmixord} and Fig.~\ref{fig:sig_bd}.}
\label{tb:comalgmixord}
\centering
\resizebox{2.04\columnwidth}{!}{
\begin{tabular}{c|l|ccccccccccr}
\toprule
Index & Data Set & KMD & ECC & WKM & MWKM & SBC & WOC & CDE & UNTIE & DLC & HD-NDW & $\Delta$\ \ \ \ \ \\
\midrule
\midrule
\multirow{11}{*}{ARI}
& Assistant& 0.111$\pm$0.06	& 0.133$\pm$0.09	& 0.113$\pm$0.08	& 0.138$\pm$0.09	& 0.153$\pm$0.04	& \underline{0.194$\pm$0.08}	& 0.131$\pm$0.06	& 0.152$\pm$0.04	& 0.152$\pm$0.09	& \textbf{0.330$\pm$0.05} &70.1\%\\	
& Lenses& 0.088$\pm$0.13	& 0.104$\pm$0.14	& 0.087$\pm$0.17	& 0.124$\pm$0.13	& \underline{0.148$\pm$0.11}	& 0.117$\pm$0.16	& 0.085$\pm$0.15	& 0.088$\pm$0.13	& 0.146$\pm$0.10	& \textbf{0.227$\pm$0.21} &53.4\%\\	
& Cancer& 0.018$\pm$0.05	& 0.050$\pm$0.07	& 0.014$\pm$0.04	& 0.056$\pm$0.06	& 0.083$\pm$0.08	& 0.076$\pm$0.07	& 0.083$\pm$0.08	& \underline{0.085$\pm$0.11}	& 0.035$\pm$0.05	& \textbf{0.090$\pm$0.10} &6.5\%\\	
& Hayes& -0.001$\pm$0.03	& 0.017$\pm$0.05	& 0.020$\pm$0.02	& 0.016$\pm$0.01	& -0.012$\pm$0.01	& 0.019$\pm$0.04	& 0.081$\pm$0.04	& \underline{0.084$\pm$0.06}	& 0.026$\pm$0.03	& \textbf{0.091$\pm$0.03} &8.8\%\\	
& Lym& 0.108$\pm$0.04	& 0.194$\pm$0.04	& 0.075$\pm$0.05	& 0.131$\pm$0.05	& 0.127$\pm$0.07	& 0.163$\pm$0.06	& 0.193$\pm$0.03	& \underline{0.197$\pm$0.05}	& \textbf{0.200$\pm$0.06}	& 0.195$\pm$0.03 &-2.4\%\\	
& Nursery& 0.054$\pm$0.02	& 0.072$\pm$0.10	& 0.083$\pm$0.11	& 0.058$\pm$0.02	& 0.017$\pm$0.01	& 0.002$\pm$0.00	& 0.053$\pm$0.02	& 0.084$\pm$0.02	& \underline{0.115$\pm$0.08}	& \textbf{0.133$\pm$0.07} &15.8\%\\	
& Photo& 0.102$\pm$0.06	& 0.121$\pm$0.09	& 0.100$\pm$0.08	& 0.140$\pm$0.09	& 0.186$\pm$0.05	& 0.158$\pm$0.09	& 0.115$\pm$0.07	& 0.115$\pm$0.09	& \underline{0.267$\pm$0.07}	& \textbf{0.318$\pm$0.06} &19.3\%\\	
& Lecturer& 0.034$\pm$0.02	& 0.035$\pm$0.02	& 0.032$\pm$0.02	& 0.038$\pm$0.01	& 0.046$\pm$0.01	& 0.040$\pm$0.02	& 0.034$\pm$0.02	& 0.033$\pm$0.02	& \underline{0.151$\pm$0.01}	& \textbf{0.154$\pm$0.01} &1.5\%\\	
& Social& 0.043$\pm$0.02	& 0.059$\pm$0.02	& 0.043$\pm$0.01	& 0.047$\pm$0.02	& 0.093$\pm$0.02	& 0.036$\pm$0.02	& 0.068$\pm$0.02	& 0.071$\pm$0.02	& \underline{0.108$\pm$0.00}	& \textbf{0.112$\pm$0.01} &3.2\%\\	
& Selection& 0.151$\pm$0.04	& 0.181$\pm$0.03	& 0.173$\pm$0.03	& 0.171$\pm$0.03	& 0.200$\pm$0.01	& 0.183$\pm$0.04	& 0.219$\pm$0.03	& 0.221$\pm$0.03	& \underline{0.313$\pm$0.01}	& \textbf{0.328$\pm$0.02} &5.1\%\\	
& Car& 0.025$\pm$0.04	& 0.058$\pm$0.05	& 0.026$\pm$0.04	& 0.031$\pm$0.02	& 0.027$\pm$0.03	& 0.035$\pm$0.03	& 0.019$\pm$0.05	& 0.023$\pm$0.06	& \underline{0.112$\pm$0.02}	& \textbf{0.128$\pm$0.04} &14.5\%\\	
\midrule	
\multicolumn{2}{c|}{Averaged Rank}	& 8.55	& 5.82	& 8.18	& 6.18	& 5.09	& 5.64	& 6.45	& 4.91	& 3.00	& \textbf{1.18}	\\						
\midrule
\multirow{11}{*}{NMI}
& Assistant& 0.152$\pm$0.07	& 0.182$\pm$0.10	& 0.160$\pm$0.10	& 0.172$\pm$0.10	& 0.184$\pm$0.05	& \underline{0.262$\pm$0.09}	& 0.159$\pm$0.07	& 0.188$\pm$0.06	& 0.212$\pm$0.11	& \textbf{0.390$\pm$0.04} &48.8\%\\	
& Lenses& 0.227$\pm$0.10	& 0.255$\pm$0.14	& 0.199$\pm$0.18	& 0.276$\pm$0.12	& 0.305$\pm$0.07	& 0.262$\pm$0.14	& 0.203$\pm$0.14	& 0.213$\pm$0.13	& \underline{0.308$\pm$0.09}	& \textbf{0.342$\pm$0.16} &11.3\%\\	
& Cancer& 0.011$\pm$0.02	& 0.029$\pm$0.04	& 0.008$\pm$0.02	& 0.024$\pm$0.03	& 0.040$\pm$0.03	& 0.034$\pm$0.03	& 0.045$\pm$0.04	& \underline{0.046$\pm$0.05}	& 0.014$\pm$0.02	& \textbf{0.062$\pm$0.03} &37.1\%\\	
& Hayes& 0.019$\pm$0.04	& 0.032$\pm$0.05	& 0.026$\pm$0.02	& 0.033$\pm$0.03	& 0.003$\pm$0.01	& 0.043$\pm$0.06	& \underline{0.087$\pm$0.03}	& 0.086$\pm$0.05	& 0.032$\pm$0.03	& \textbf{0.103$\pm$0.03} &18.3\%\\	
& Lym& 0.168$\pm$0.04	& 0.243$\pm$0.04	& 0.130$\pm$0.05	& 0.188$\pm$0.05	& 0.170$\pm$0.04	& 0.231$\pm$0.06	& 0.237$\pm$0.04	& \underline{0.243$\pm$0.05}	& 0.223$\pm$0.05	& \textbf{0.258$\pm$0.03} &5.9\%\\	
& Nursery& 0.059$\pm$0.02	& 0.103$\pm$0.13	& 0.105$\pm$0.13	& 0.103$\pm$0.03	& 0.032$\pm$0.02	& 0.006$\pm$0.00	& 0.056$\pm$0.02	& 0.101$\pm$0.03	& \underline{0.117$\pm$0.11}	& \textbf{0.162$\pm$0.09} &39.2\%\\	
& Photo& 0.143$\pm$0.06	& 0.177$\pm$0.09	& 0.151$\pm$0.10	& 0.180$\pm$0.10	& 0.221$\pm$0.05	& 0.222$\pm$0.11	& 0.181$\pm$0.08	& 0.200$\pm$0.10	& \underline{0.339$\pm$0.03}	& \textbf{0.373$\pm$0.03} &10.1\%\\	
& Lecturer& 0.054$\pm$0.02	& 0.059$\pm$0.02	& 0.057$\pm$0.02	& 0.060$\pm$0.02	& 0.073$\pm$0.02	& 0.064$\pm$0.03	& 0.056$\pm$0.02	& 0.059$\pm$0.02	& \underline{0.215$\pm$0.01}	& \textbf{0.217$\pm$0.01} &0.8\%\\	
& Social& 0.065$\pm$0.02	& 0.086$\pm$0.02	& 0.060$\pm$0.02	& 0.068$\pm$0.02	& 0.131$\pm$0.02	& 0.059$\pm$0.02	& 0.094$\pm$0.02	& 0.088$\pm$0.01	& \underline{0.167$\pm$0.00}	& \textbf{0.168$\pm$0.01} &0.2\%\\	
& Selection& 0.280$\pm$0.04	& 0.335$\pm$0.03	& 0.305$\pm$0.03	& 0.308$\pm$0.02	& 0.353$\pm$0.01	& 0.307$\pm$0.04	& 0.370$\pm$0.02	& 0.368$\pm$0.02	& \underline{0.491$\pm$0.01}	& \textbf{0.510$\pm$0.01} &3.7\%\\	
& Car& 0.047$\pm$0.02	& 0.121$\pm$0.07	& 0.062$\pm$0.05	& 0.064$\pm$0.03	& 0.071$\pm$0.04	& 0.079$\pm$0.06	& 0.091$\pm$0.07	& 0.106$\pm$0.06	& \underline{0.219$\pm$0.01}	& \textbf{0.228$\pm$0.01} &3.9\%\\	
\midrule	
\multicolumn{2}{c|}{Averaged Rank}	& 9.00	& 5.55	& 8.45	& 6.27	& 5.55	& 5.64	& 5.64	& 4.55	& 3.36	& \textbf{1.00}	\\		
\midrule
\multirow{11}{*}{CA}
& Assistant& 0.522$\pm$0.07	& 0.536$\pm$0.08	& 0.527$\pm$0.09	& 0.546$\pm$0.09	& 0.568$\pm$0.08	& \underline{0.621$\pm$0.07}	& 0.531$\pm$0.06	& 0.549$\pm$0.05	& 0.570$\pm$0.09	& \textbf{0.639$\pm$0.07} &2.9\%\\	
& Lenses& 0.534$\pm$0.09	& 0.537$\pm$0.11	& 0.538$\pm$0.10	& 0.557$\pm$0.10	& \underline{0.564$\pm$0.09}	& 0.544$\pm$0.11	& 0.512$\pm$0.10	& 0.513$\pm$0.07	& 0.561$\pm$0.07	& \textbf{0.588$\pm$0.13} &4.1\%\\	
& Cancer& 0.564$\pm$0.06	& 0.586$\pm$0.09	& 0.536$\pm$0.07	& 0.614$\pm$0.08	& 0.624$\pm$0.11	& 0.629$\pm$0.10	& \underline{0.630$\pm$0.10}	& \underline{0.630$\pm$0.10}	& 0.584$\pm$0.08	& \textbf{0.651$\pm$0.09} &3.4\%\\	
& Hayes& 0.384$\pm$0.03	& 0.414$\pm$0.06	& 0.439$\pm$0.05	& 0.416$\pm$0.02	& 0.354$\pm$0.02	& 0.413$\pm$0.08	& 0.442$\pm$0.05	& \underline{0.452$\pm$0.04}	& 0.446$\pm$0.05	& \textbf{0.487$\pm$0.05} &7.7\%\\	
& Lym& 0.462$\pm$0.05	& 0.512$\pm$0.04	& 0.433$\pm$0.07	& 0.482$\pm$0.06	& 0.505$\pm$0.06	& \underline{0.551$\pm$0.05}	& 0.519$\pm$0.04	& 0.550$\pm$0.05	& 0.538$\pm$0.06	& \textbf{0.601$\pm$0.07} &9.2\%\\	
& Nursery& 0.378$\pm$0.04	& 0.368$\pm$0.07	& 0.395$\pm$0.09	& 0.359$\pm$0.03	& 0.323$\pm$0.03	& 0.292$\pm$0.02	& 0.366$\pm$0.00	& 0.397$\pm$0.01	& \underline{0.404$\pm$0.04}	& \textbf{0.423$\pm$0.06} &4.8\%\\	
& Photo& 0.511$\pm$0.07	& 0.524$\pm$0.08	& 0.517$\pm$0.09	& 0.553$\pm$0.09	& 0.557$\pm$0.05	& 0.584$\pm$0.09	& 0.501$\pm$0.06	& 0.500$\pm$0.09	& \underline{0.668$\pm$0.06}	& \textbf{0.698$\pm$0.05} &4.4\%\\	
& Lecturer& 0.335$\pm$0.03	& 0.328$\pm$0.03	& 0.319$\pm$0.03	& 0.322$\pm$0.03	& 0.339$\pm$0.02	& 0.331$\pm$0.04	& 0.319$\pm$0.03	& 0.338$\pm$0.05	& \underline{0.455$\pm$0.04}	& \textbf{0.465$\pm$0.04} &2.2\%\\	
& Social& 0.370$\pm$0.03	& 0.384$\pm$0.03	& 0.371$\pm$0.02	& 0.372$\pm$0.03	& \underline{0.421$\pm$0.02}	& 0.372$\pm$0.03	& 0.391$\pm$0.02	& 0.409$\pm$0.02	& 0.414$\pm$0.02	& \textbf{0.453$\pm$0.04} &7.6\%\\	
& Selection& 0.365$\pm$0.04	& 0.372$\pm$0.04	& 0.373$\pm$0.03	& 0.369$\pm$0.04	& 0.386$\pm$0.01	& 0.427$\pm$0.05	& 0.407$\pm$0.03	& 0.437$\pm$0.04	& \textbf{0.505$\pm$0.03}	& \underline{0.489$\pm$0.03} &-3.2\%\\	
& Car& 0.370$\pm$0.04	& 0.384$\pm$0.06	& 0.375$\pm$0.07	& 0.369$\pm$0.03	& 0.357$\pm$0.04	& 0.366$\pm$0.05	& 0.389$\pm$0.05	& 0.390$\pm$0.05	& \underline{0.437$\pm$0.04}	& \textbf{0.453$\pm$0.06} &3.5\%\\	
\midrule	
\multicolumn{2}{c|}{Averaged Rank}	& 8.18	& 6.55	& 7.45	& 6.91	& 5.64	& 5.45	& 6.32	& 4.41	& 3.00	& \textbf{1.09}	\\						
\bottomrule
\end{tabular}}
\end{table*}

\begin{table*}[t]
\caption{Clustering performance of various clustering algorithms on nominal data sets. The column of `$\Delta$' reports the improvements achieved by HD-NDW in comparison with the best-performing counterparts on different data sets.}
\label{tb:comalgnom}
\centering
\resizebox{1.9\columnwidth}{!}{
\begin{tabular}{c|l|cccccccccr}
\toprule
Index & Data Set & KMD & ECC & WKM & MWKM &SBC & WOC & CDE & UNTIE & HD-NDW & $\Delta$\ \ \ \ \ \\
\midrule
\multirow{4}{*}{ARI}
& Solar& 0.223$\pm$0.06	& 0.194$\pm$0.06	& 0.132$\pm$0.09	& 0.199$\pm$0.06	& 0.126$\pm$0.03	& 0.229$\pm$0.10	& 0.237$\pm$0.08	& \underline{0.255$\pm$0.10}	& \textbf{0.318$\pm$0.08} &24.8\%\\	
& Zoo& 0.628$\pm$0.18	& 0.530$\pm$0.15	& 0.651$\pm$0.18	& 0.594$\pm$0.18	& 0.413$\pm$0.14	& 0.618$\pm$0.13	& \underline{0.741$\pm$0.11}	& \textbf{0.748$\pm$0.13}	& 0.721$\pm$0.15 &-3.6\%\\	
& Voting& 0.520$\pm$0.02	& 0.544$\pm$0.01	& 0.535$\pm$0.00	& 0.542$\pm$0.01	& \underline{0.560$\pm$0.03}	& 0.537$\pm$0.00	& 0.534$\pm$0.08	& 0.558$\pm$0.07	& \textbf{0.564$\pm$0.00} &0.6\%\\	
& Soybean& 0.688$\pm$0.22	& 0.659$\pm$0.16	& 0.772$\pm$0.21	& 0.740$\pm$0.22	& 0.816$\pm$0.11	& 0.788$\pm$0.21	& \underline{0.821$\pm$0.19}	& \textbf{0.829$\pm$0.17}	& 0.803$\pm$0.21 &-3.1\%\\	
\midrule	
\multicolumn{2}{c|}{Averaged Rank}	& 6.75	& 7.00	& 6.25	& 6.25	& 5.75	& 5.25	& 3.75	& \textbf{1.75}	& 2.25	\\		
\midrule
\multirow{4}{*}{NMI}
& Solar& 0.300$\pm$0.05	& 0.278$\pm$0.06	& 0.218$\pm$0.10	& 0.271$\pm$0.06	& 0.196$\pm$0.03	& 0.331$\pm$0.09	& 0.319$\pm$0.08	& \underline{0.348$\pm$0.10}	& \textbf{0.408$\pm$0.08} &17.2\%\\	
& Zoo& 0.753$\pm$0.09	& 0.700$\pm$0.06	& 0.779$\pm$0.08	& 0.745$\pm$0.08	& 0.595$\pm$0.08	& 0.786$\pm$0.05	& \textbf{0.810$\pm$0.05}	& 0.808$\pm$0.08	& \underline{0.809$\pm$0.08} &-0.1\%\\	
& Voting& 0.448$\pm$0.02	& \underline{0.476$\pm$0.01}	& 0.452$\pm$0.01	& 0.473$\pm$0.01	& 0.473$\pm$0.00	& 0.475$\pm$0.00	& 0.462$\pm$0.07	& 0.458$\pm$0.08	& \textbf{0.489$\pm$0.00} &2.7\%\\	
& Soybean& 0.805$\pm$0.15	& 0.771$\pm$0.10	& 0.849$\pm$0.12	& 0.847$\pm$0.13	& 0.856$\pm$0.06	& 0.885$\pm$0.11	& 0.892$\pm$0.11	& \textbf{0.902$\pm$0.10}	& \underline{0.897$\pm$0.11} &-0.6\%\\	
\midrule	
\multicolumn{2}{c|}{Averaged Rank}	& 7.00	& 6.25	& 6.75	& 6.50	& 6.75	& 3.50	& 3.50	& 3.25	& \textbf{1.50}	\\		
\midrule
\multirow{4}{*}{CA}
& Solar& 0.482$\pm$0.05	& 0.442$\pm$0.05	& 0.400$\pm$0.06	& 0.462$\pm$0.05	& 0.400$\pm$0.04	& 0.483$\pm$0.07	& 0.483$\pm$0.07	& \underline{0.498$\pm$0.06}	& \textbf{0.540$\pm$0.05} &8.3\%\\	
& Zoo& 0.676$\pm$0.13	& 0.623$\pm$0.10	& 0.703$\pm$0.12	& 0.647$\pm$0.13	& 0.554$\pm$0.09	& 0.669$\pm$0.10	& 0.758$\pm$0.08	& \textbf{0.778$\pm$0.09}	& \underline{0.760$\pm$0.10} &-2.4\%\\	
& Voting& 0.861$\pm$0.01	& 0.869$\pm$0.01	& 0.852$\pm$0.00	& 0.868$\pm$0.00	& \underline{0.875$\pm$0.00}	& 0.867$\pm$0.00	& 0.864$\pm$0.04	& 0.872$\pm$0.05	& \textbf{0.876$\pm$0.00} &0.1\%\\	
& Soybean& 0.791$\pm$0.17	& 0.773$\pm$0.14	& 0.837$\pm$0.17	& 0.811$\pm$0.17	& \underline{0.874$\pm$0.10}	& 0.821$\pm$0.17	& 0.874$\pm$0.14	& \textbf{0.876$\pm$0.13}	& 0.849$\pm$0.16 &-3.1\%\\	
\midrule	
\multicolumn{2}{c|}{Averaged Rank}	& 6.50	& 7.00	& 6.75	& 6.25	& 5.25	& 5.50	& 4.00	& \textbf{1.75}	& 2.00	\\						
\bottomrule
\end{tabular}}
\end{table*}

\begin{table}[t]
\caption{Wilcoxon signed-rank test on the performance of HD-NDW vs. DLC and HD-NDW vs. UNTIE. The symbol ``+'' indicates that HD-NDW is significantly different from a certain counterpart for the two-tailed Wilcoxon signed-rank test at confidence interval 99\% (i.e., $\alpha$ = 0.01).}
\label{tb:sigalgmixord}
\centering
\begin{tabular}{c|c|c}
\toprule
\rule{-2.1pt}{9pt}
Index & HD-NDW vs. DLC & HD-NDW vs. UNTIE \\
\midrule
ARI   & +	& + \\	
NMI   & +	& + \\	
CA    & +	& + \\			
\bottomrule
\end{tabular}
\end{table}

\begin{figure}[t]
\newcommand{\mylwd}{0.495}
\newcommand{\mydmwd}{1.7in}
\begin{minipage}{1\linewidth}
  \centerline{\includegraphics[width=3.47in]{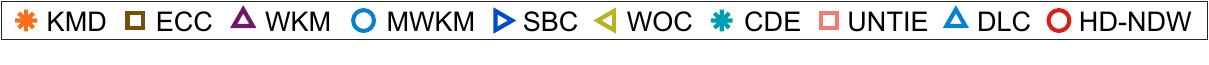}}
\end{minipage}
\vfill
\begin{minipage}{\mylwd\linewidth}
  \centerline{\includegraphics[width=\mydmwd]{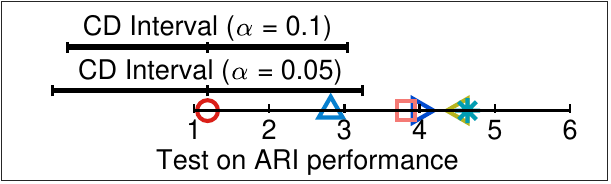}}
\end{minipage}
\hfill
\begin{minipage}{\mylwd\linewidth}
  \centerline{\includegraphics[width=\mydmwd]{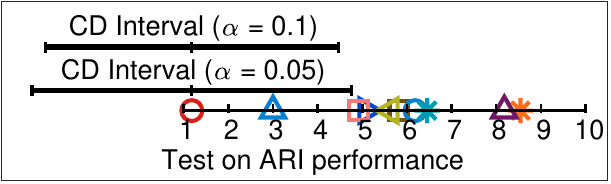}}
\end{minipage}
\vfill
\begin{minipage}{\mylwd\linewidth}
  \centerline{\includegraphics[width=\mydmwd]{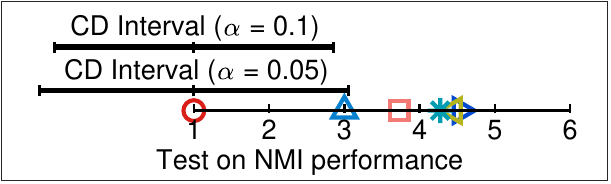}}
\end{minipage}
\hfill
\begin{minipage}{\mylwd\linewidth}
  \centerline{\includegraphics[width=\mydmwd]{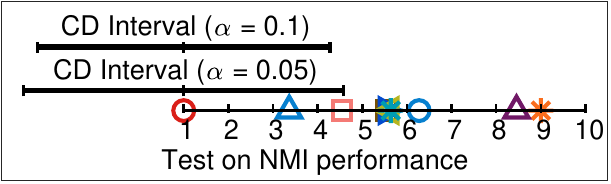}}
\end{minipage}
\vfill
\begin{minipage}{\mylwd\linewidth}
  \centerline{\includegraphics[width=\mydmwd]{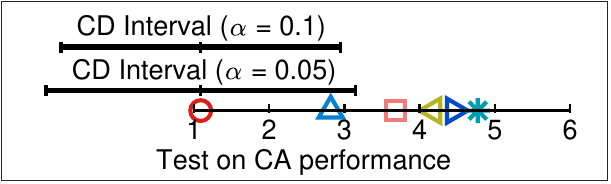}}
  \centerline{\scriptsize{(a) Test for recent-five-year methods.}}
\end{minipage}
\hfill
\begin{minipage}{\mylwd\linewidth}
  \centerline{\includegraphics[width=\mydmwd]{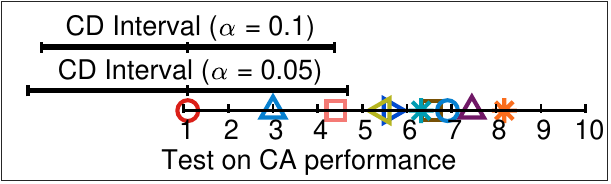}}
  \centerline{\scriptsize{(b) Test for all the compared methods.}}
\end{minipage}
\caption{Bonferroni-Dunn (BD) test on the performance of (a) methods proposed in recent five years, and (b) all the compared methods. Critical Difference (CD) for the two-tailed BD tests in (a) at confidence interval 95\% ($\alpha$ = 0.05) and 90\% ($\alpha$ = 0.1) are 2.05 and 1.86, respectively. CD for the two-tailed BD tests in (b) at confidence interval 95\% ($\alpha$ = 0.05) and 90\% ($\alpha$ = 0.1) are 3.58 and 3.28, respectively. The counterparts rank outside the CD intervals are believed to be significantly different from HD-NDW.}
\label{fig:sig_bd}
\end{figure}

Since a key working principle of HD-NDW is to convert the nominal attributes into ordinal ones for more reasonable distance measurement, the superiority of HD-NDW will be more prominent on mixed and ordinal data sets. In order to conduct a more targeted evaluation, we report the clustering performance on mixed and ordinal data sets in Table~\ref{tb:comalgmixord}. To ensure the completeness of the evaluation, the performance on nominal data sets is reported in Table~\ref{tb:comalgnom}. The best and second-best results are highlighted using boldface and underline, respectively. Improvements achieved by HD-NDW in comparison with the best-performing counterparts on different data sets are reported in the column of `$\Delta$'. For each data set, the compared methods are ranked according to their performance, and the averaged rank of each method is reported. From the results shown in Table~\ref{tb:comalgmixord} and \ref{tb:comalgnom}, we have the following observations:
\begin{itemize}
\item HD-NDW obviously outperforms the other counterparts on mixed categorical data sets, because the homogeneous distance definition and the distance weighting mechanism may have desired effects on mixed categorical data sets.
\item HD-NDW and DLC significantly outperform the other counterparts on ordinal data sets, because they take into account the intra- and inter-attribute order relationship, by which the learned distances are more appropriate for clustering.
\item In the comparison on nominal data sets, superiority of HD-NDW is not as significant as on mixed and ordinal data sets because the HD component that uniformly defines distances for ordinal and nominal attributes will not have desired impact when processing nominal data. Nevertheless, since NDW still acts in booting the clustering performance, HD-NDW is still competitive in comparison with the state-of-the-art UNTIE, and obviously outperforms the others.
\item Although UNTIE is not specially designed for representing data set with ordinal attributes, it still shows strong data representation ability, because it performs the best in comparison with the counterparts except the two methods (i.e., DLC and HD-NDW) that contain specially designed mechanisms for exploiting the information embedded in ordinal attributes. As for the performance on nominal data sets, UNTIE performs the best in general, while HD-NDW is still very competitive.
\end{itemize}

\subsubsection{Significance Test}

According to the averaged rank shown in Table~\ref{tb:comalgmixord}, UNTIE and DLC are clearly the two most competitive counterparts. We conduct significance test using Wilcoxon signed-rank test and report the results in Table~\ref{tb:sigalgmixord}. It can be seen that even at 99\% confidence interval, HD-NDW is still significantly better than the two counterparts in terms of all three validity indices.

To intuitively compare the proposed HD-NDW with the other counterparts, we further perform Bonferroni-Dunn test \cite{r1test} on the performance of different methods and visualize the results in Fig.~\ref{fig:sig_bd}. The counterparts rank outside the Critical Difference (CD) intervals are believed to be significantly different from HD-NDW. It can be observed from Fig.~\ref{fig:sig_bd}(a) that HD-NDW is significantly better than almost all five methods proposed in recent five years at confidence interval 90\%. In comparison with all nine counterparts in Fig.~\ref{fig:sig_bd}(b), HD-NDW is still significantly better than eight counterparts at confidence interval 90\%. Although HD-NDW is not significantly better than DLC, HD-NDW is capable in processing any-type categorical data, while DLC is designed for ordinal data only, which cannot be directly used for mixed data and is incompetent for nominal data. Therefore, HD-NDW also demonstrates superiority in comparison with DLC in terms of availability for nominal data clustering.

\subsubsection{Visualization of Data Representation}

\begin{figure}[t]
\newcommand{\mylwd}{0.325}
\newcommand{\mydmwd}{1.1in}
\begin{minipage}{\mylwd\linewidth}
  \centerline{\includegraphics[width=\mydmwd]{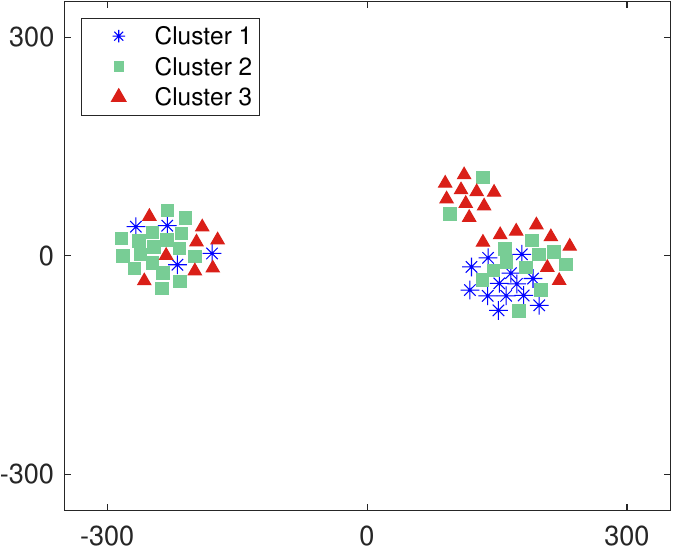}}
  \centerline{\scriptsize{UNTIE Representation}}
\end{minipage}
\hfill
\begin{minipage}{\mylwd\linewidth}
  \centerline{\includegraphics[width=\mydmwd]{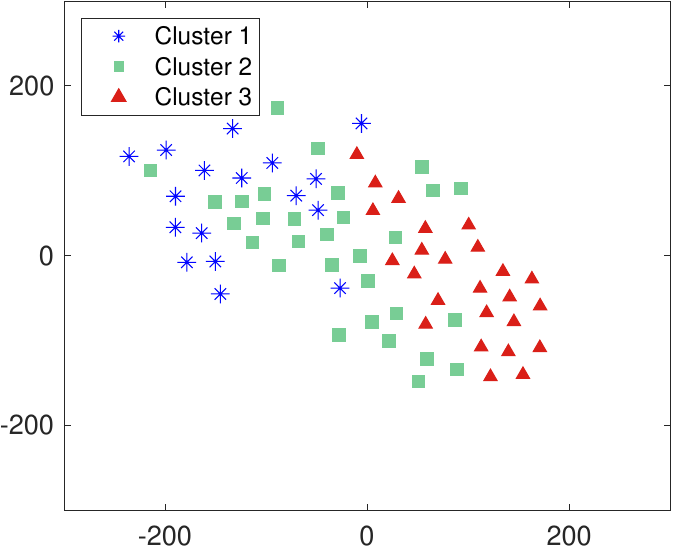}}
  \centerline{\scriptsize{DLC Representation}}
\end{minipage}
\hfill
\begin{minipage}{\mylwd\linewidth}
  \centerline{\includegraphics[width=\mydmwd]{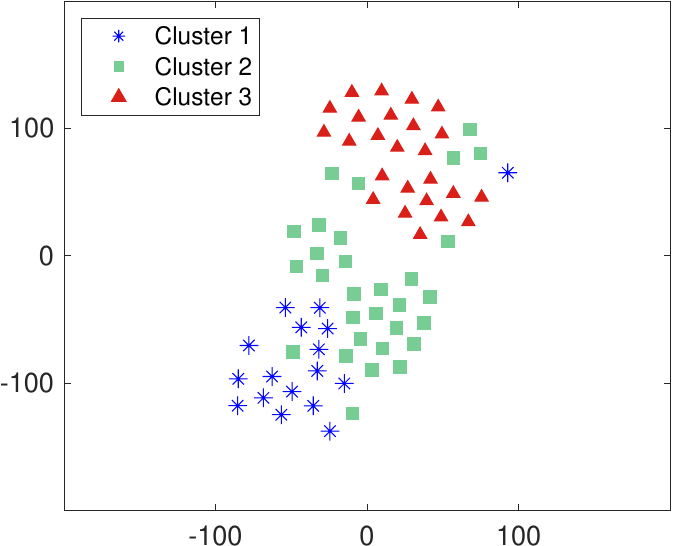}}
  \centerline{\scriptsize{HD-NDW Representation}}
\end{minipage}
\caption{t-SNE visualization of the representations produced by UNTIE, DLC, and HD-NDW on Assistant data set. The three types of markers indicate data objects belonging to different true clusters.}
\label{fig:tsne}
\end{figure}

To intuitively compare the reasonableness of the learned representation or distances of the three best performing methods, i.e., UNTIE, DLC, and HD-NDW, we visualize their representations in Fig.~\ref{fig:tsne} by converting them into two-dimensional points using t-Distributed Stochastic Neighbor Embedding (t-SNE) \cite{r1visual}. Since DLC and HD-NDW are not representation-based methods, we first use them to learn intra-attribute distances, and then encode the data values using the learned distances for representation. Since the distances learned by DLC satisfy $\text{dist}(o_a,o_b)+\text{dist}(o_b,o_c)=\text{dist}(o_a,o_c)$ for $a<b<c$ or $a>b>c$, we directly encode the possible values by $o_1=0$, $o_2=\text{dist}(o_1,o_2)$, $o_3=\text{dist}(o_1,o_3)$, and so on, which will not twist the distances learned by DLC. For the distances learned by HD-NDW, we encode a possible value using the distances between it and all the intra-attribute possible values to preserve the information of the learned distances. For example, for an attribute with possible values \{$o_a$, $o_b$, $o_c$\}, the value $o_b$ is encoded into a vector $[\text{dist}(o_a,o_b),\text{dist}(o_b,o_b),\text{dist}(o_c,o_b)]^\top$ by the HD-NDW learned distance metric. Note that the HD-NDW distance here is the one defined by HD multiplied by the corresponding distance weight learned by HD-NDW.

It can be observed that the true clusters in the HD-NDW-represented data set are obviously more separable in comparison with UNTIE and DLC. The reason should be that Assistant is a mixed categorical data set that is composed of nominal and ordinal attributes. For this kind of data, UNTIE is unable to take into account the order information embedded in ordinal attributes, while DLC is unsuitable for learning distances of nominal attributes.

\subsubsection{Visualization of Cluster Discrimination}

\begin{figure}[t]
\newcommand{\mylwd}{1}
  \centerline{\includegraphics[width=3.45in]{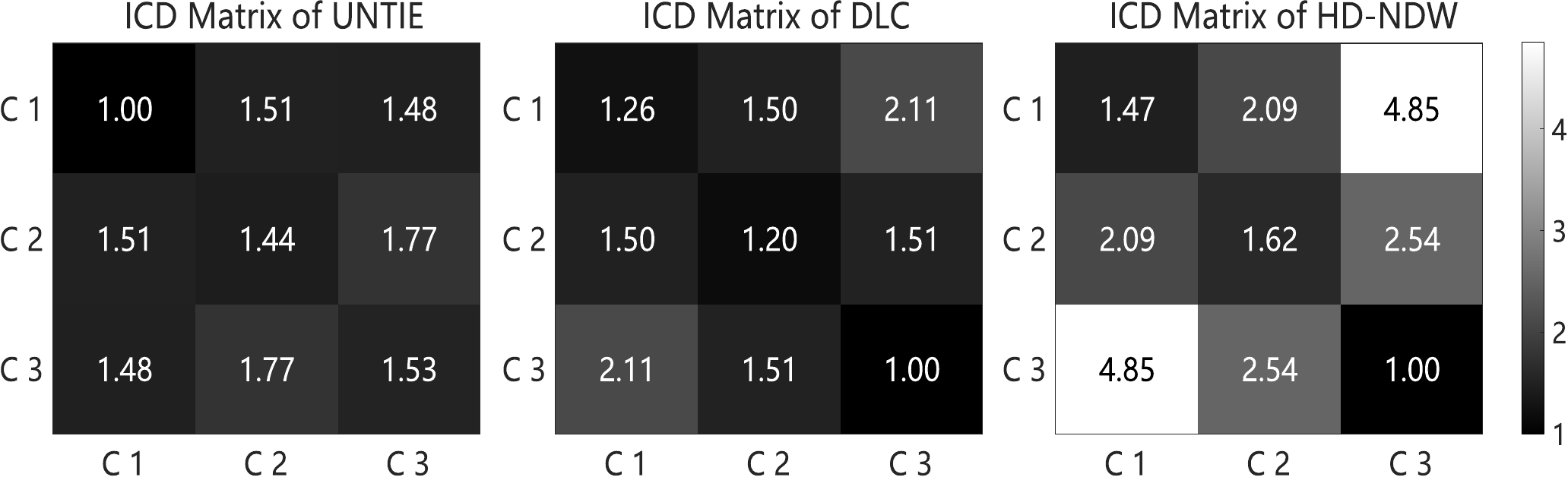}}
\caption{Gray scale maps of the ICD matrices produced by UNTIE, DLC, and HD-NDW on Assistant data set. Darker on the main diagonal and lighter on the other locations indicate a better distance metric.}
\label{fig:icdm}
\end{figure}

Averaged ICD computed based on the true cluster labels of a data set can intuitively indicate the discrimination ability of a distance metric. According to \cite{adm}, averaged ICD between two clusters $C_l$ and $C_t$ with $n_l$ and $n_t$ data objects, respectively, is computed by $\sum_{\textbf{x}_i\in C_l}\sum_{\textbf{x}_j\in C_t}\text{dist}(\textbf{x}_i,\textbf{x}_j)/(n_ln_t)$. When $l=t$, it computes the averaged intra-attribute distance; otherwise, it computes the averaged inter-attribute distance. Since different distance metrics may have different scales, computing the multiple relationship between the averaged intra- and inter-cluster distances is a feasible solution \cite{jdm} to fairly compare the discrimination ability of different metrics. Therefore, we pre-process the ICD matrix of each distance metric by dividing all the values in the matrix by the minimum value in this matrix. Then we visualize the pre-processed ICD matrices as gray scale maps in Fig.~\ref{fig:icdm}. ICD matrix of a better distance metric should be darker on the main diagonal and lighter on the other locations, which indicates smaller averaged intra-cluster distances and larger averaged inter-cluster distances, respectively. From Fig.~\ref{fig:icdm}, it is clear that HD-NDW has better cluster discrimination ability than UNTIE and DLC.

\subsection{Effectiveness Evaluation of HD}\label{subsct:ex_metric}

\begin{figure}[t]
\newcommand{\mylwd}{0.105}
\newcommand{\mydmwd}{0.36in}
\newcommand{\myvsps}{0.1pt}
\newcommand{\myvspshead}{-8pt}
\begin{minipage}{0.05\linewidth}
  \centerline{\tiny $A^1$}
\end{minipage}
\hfill
\begin{minipage}{\mylwd\linewidth}
  \centerline{\tiny Hamming}
  \centerline{\vspace{\myvspshead}}
  \centerline{\includegraphics[width=\mydmwd]{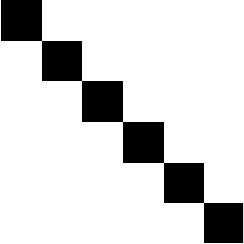}}
  \centerline{\vspace{\myvsps}}
\end{minipage}
\hfill
\begin{minipage}{\mylwd\linewidth}
  \centerline{\tiny LSM}
  \centerline{\vspace{\myvspshead}}
  \centerline{\includegraphics[width=\mydmwd]{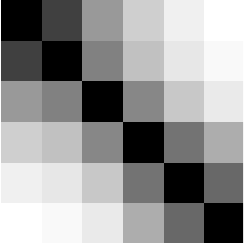}}
  \centerline{\vspace{\myvsps}}
\end{minipage}
\hfill
\begin{minipage}{\mylwd\linewidth}
  \centerline{\tiny CBDM}
  \centerline{\vspace{\myvspshead}}
  \centerline{\includegraphics[width=\mydmwd]{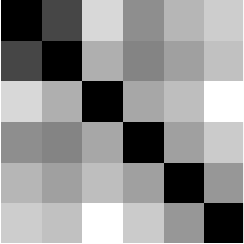}}
  \centerline{\vspace{\myvsps}}
\end{minipage}
\hfill
\begin{minipage}{\mylwd\linewidth}
  \centerline{\tiny CMS}
  \centerline{\vspace{\myvspshead}}
  \centerline{\includegraphics[width=\mydmwd]{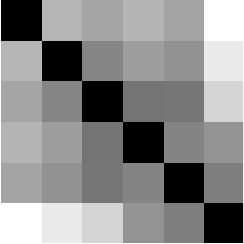}}
  \centerline{\vspace{\myvsps}}
\end{minipage}
\hfill
\begin{minipage}{\mylwd\linewidth}
  \centerline{\tiny EBDM}
  \centerline{\vspace{\myvspshead}}
  \centerline{\includegraphics[width=\mydmwd]{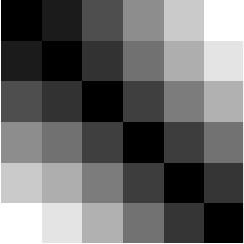}}
  \centerline{\vspace{\myvsps}}
\end{minipage}
\hfill
\begin{minipage}{\mylwd\linewidth}
  \centerline{\tiny DLC}
  \centerline{\vspace{\myvspshead}}
  \centerline{\includegraphics[width=\mydmwd]{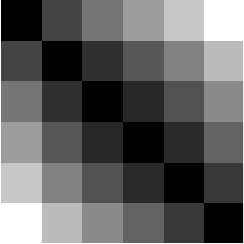}}
  \centerline{\vspace{\myvsps}}
\end{minipage}
\hfill
\begin{minipage}{\mylwd\linewidth}
  \centerline{\tiny HD}
  \centerline{\vspace{\myvspshead}}
  \centerline{\includegraphics[width=\mydmwd]{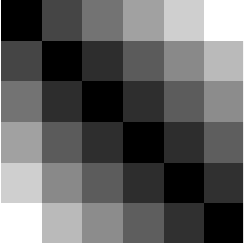}}
  \centerline{\vspace{\myvsps}}
\end{minipage}
\hfill
\begin{minipage}{\mylwd\linewidth}
  \centerline{\tiny HD-NDW}
  \centerline{\vspace{\myvspshead}}
  \centerline{\includegraphics[width=\mydmwd]{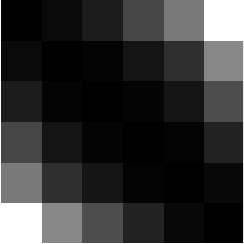}}
  \centerline{\vspace{\myvsps}}
\end{minipage}
\vfill
\begin{minipage}{0.05\linewidth}
  \centerline{\vspace{-13pt}}
  \centerline{\tiny $A^2$}
\end{minipage}
\hfill
\begin{minipage}{\mylwd\linewidth}
  \centerline{\vspace{\myvspshead}}
  \centerline{\includegraphics[width=\mydmwd]{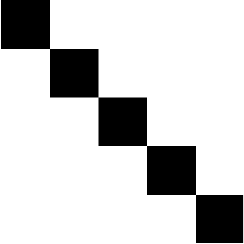}}
\end{minipage}
\hfill
\begin{minipage}{\mylwd\linewidth}
  \centerline{\vspace{\myvspshead}}
  \centerline{\includegraphics[width=\mydmwd]{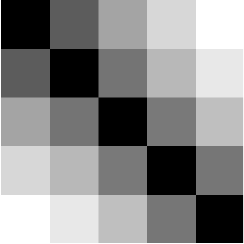}}
\end{minipage}
\hfill
\begin{minipage}{\mylwd\linewidth}
  \centerline{\vspace{\myvspshead}}
  \centerline{\includegraphics[width=\mydmwd]{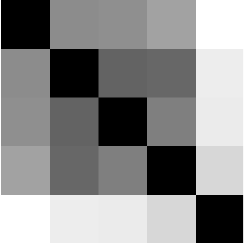}}
\end{minipage}
\hfill
\begin{minipage}{\mylwd\linewidth}
  \centerline{\vspace{\myvspshead}}
  \centerline{\includegraphics[width=\mydmwd]{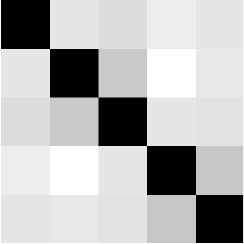}}
\end{minipage}
\hfill
\begin{minipage}{\mylwd\linewidth}
  \centerline{\vspace{\myvspshead}}
  \centerline{\includegraphics[width=\mydmwd]{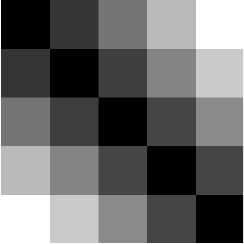}}
\end{minipage}
\hfill
\begin{minipage}{\mylwd\linewidth}
  \centerline{\vspace{\myvspshead}}
  \centerline{\includegraphics[width=\mydmwd]{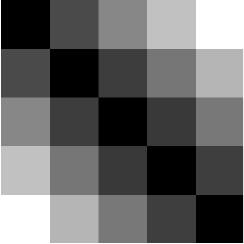}}
\end{minipage}
\hfill
\begin{minipage}{\mylwd\linewidth}
  \centerline{\vspace{\myvspshead}}
  \centerline{\includegraphics[width=\mydmwd]{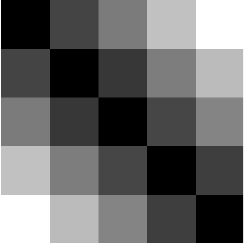}}
\end{minipage}
\hfill
\begin{minipage}{\mylwd\linewidth}
  \centerline{\vspace{\myvspshead}}
  \centerline{\includegraphics[width=\mydmwd]{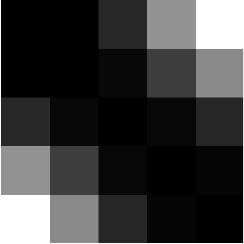}}
\end{minipage}
\caption{Gray scale maps of the intra-attribute distance matrices of the two ordinal attributes of Assistant data set produced by various distance metrics.}
\label{fig:dm_mix}
\end{figure}

In Fig.~\ref{fig:dm_mix}, we visualize the intra-attribute distances produced by different distance measures to intuitively compare them. JDM is not compared in Fig.~\ref{fig:dm_mix} because it directly measures object-cluster distance and does not produce intra-attribute distances. The produced distances are first normalized into the interval [0,1] using min-max scaling, and then the normalized distances are visualized by converting them into corresponding gray scale pixels. A lighter pixel represents a larger distance between two possible values, and a pure black pixel represents a distance value of 0. In Fig.~\ref{fig:dm_mix}, the pixel located at the $m$th column and $h$th row of a gray scale map represents the distance between the $m$th and $h$th possible values of the corresponding attribute. In general, two possible values with larger order difference should have larger distance, and thus the pixels should be darker on the main diagonal, and lighter towards the upper right and lower left corners in the gray scale maps.

It can be observed that Hamming distance is completely incapable in distinguishing the distances between different possible values. Although CBDM and CMS exploits more context information for distance measurement, they cannot reveal the order relationship among possible values of ordinal attributes. Obviously, the distances produced by LSM, EBDM, DLC, HD, and HD-NDW are consistent with the order relationship among possible values of the two ordinal attributes. Since the distances produced by HD-NDW is the weighted version of the distances produced by HD, gray scale maps of HD and HD-NDW are different in Fig.~\ref{fig:dm_mix}, but they both reflect the order relationship.

\begin{figure*}[t]
\newcommand{\mylwd}{0.19}
\newcommand{\mydmwd}{1.2in}
\begin{minipage}{\mylwd\linewidth}
  \centerline{\includegraphics[width=\mydmwd]{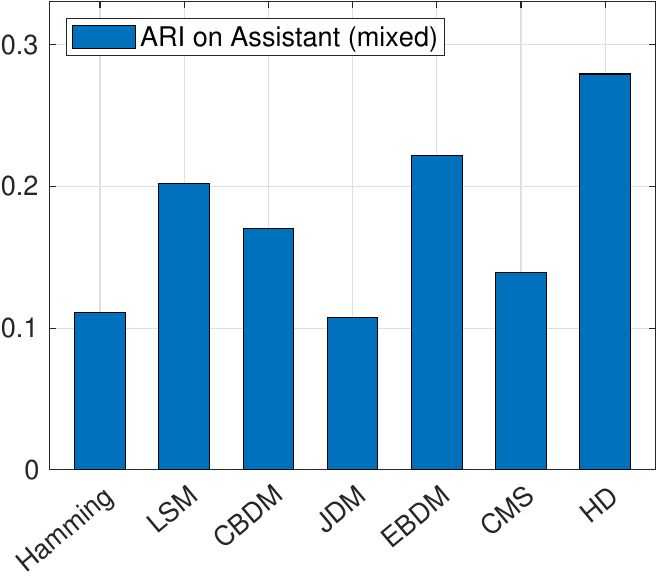}}
\end{minipage}
\hfill
\begin{minipage}{\mylwd\linewidth}
  \centerline{\includegraphics[width=\mydmwd]{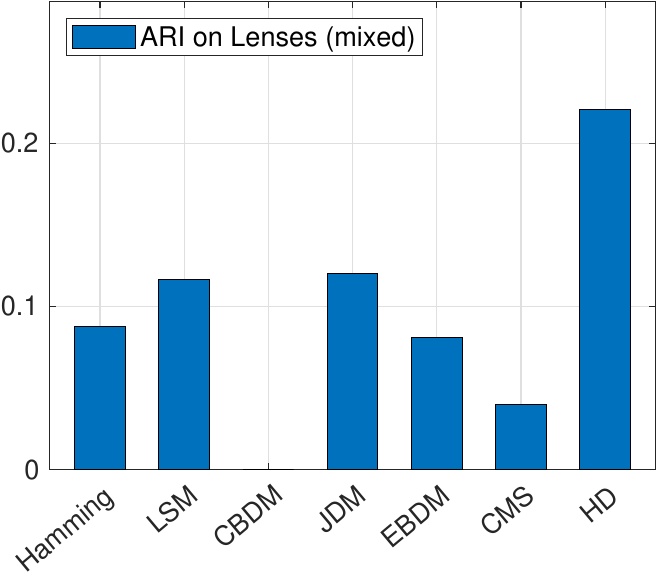}}
\end{minipage}
\hfill
\begin{minipage}{\mylwd\linewidth}
  \centerline{\includegraphics[width=\mydmwd]{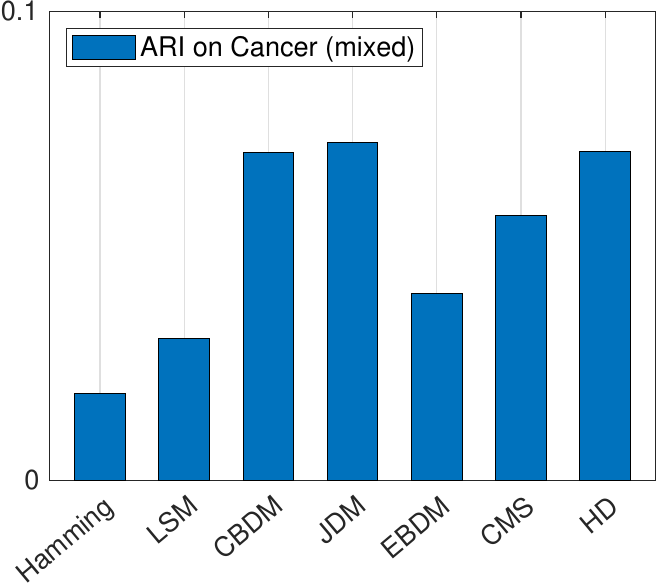}}
\end{minipage}
\hfill
\begin{minipage}{\mylwd\linewidth}
  \centerline{\includegraphics[width=\mydmwd]{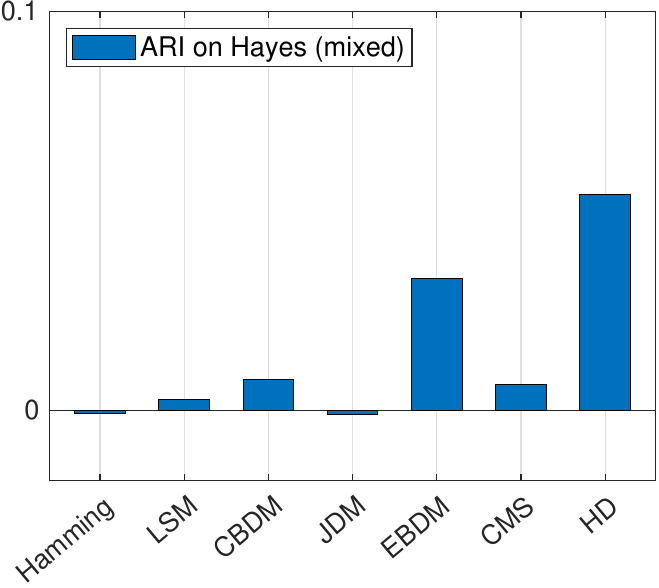}}
\end{minipage}
\hfill
\begin{minipage}{\mylwd\linewidth}
  \centerline{\includegraphics[width=\mydmwd]{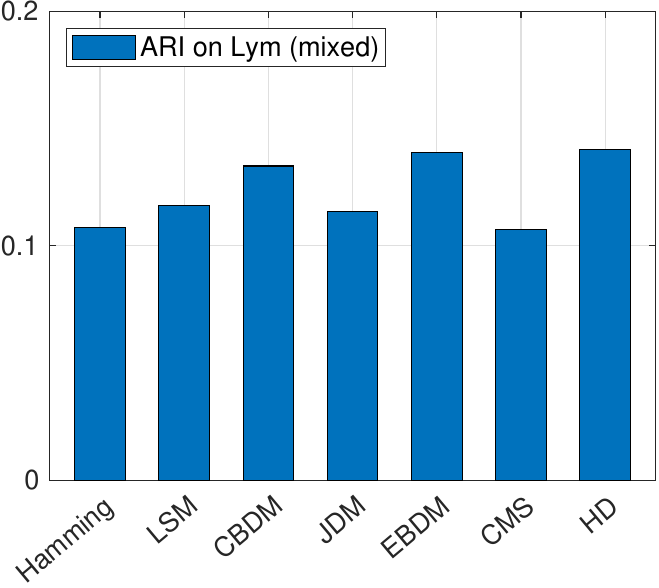}}
\end{minipage}
\vfill
\begin{minipage}{\mylwd\linewidth}
  \centerline{\includegraphics[width=\mydmwd]{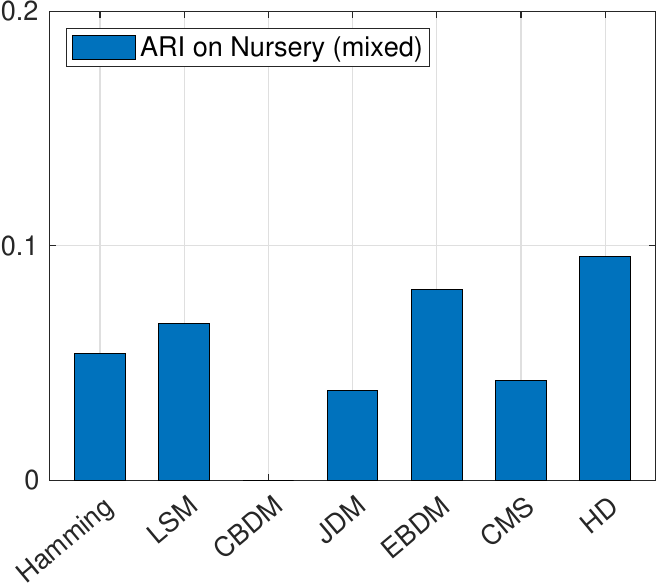}}
\end{minipage}
\hfill
\begin{minipage}{\mylwd\linewidth}
  \centerline{\includegraphics[width=\mydmwd]{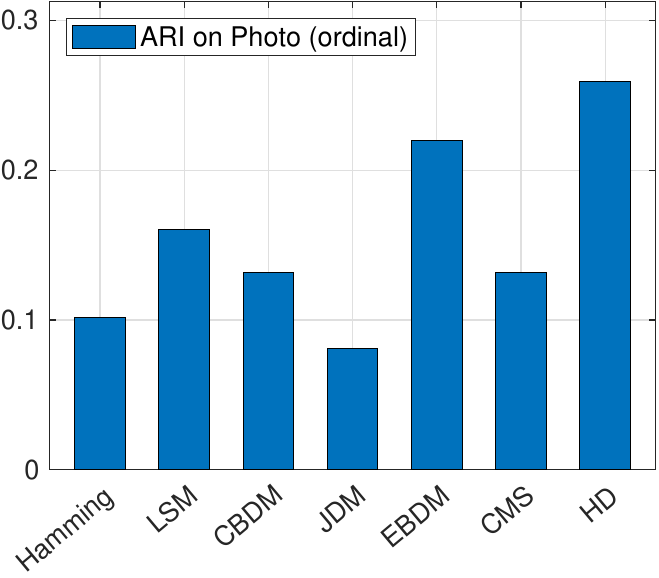}}
\end{minipage}
\hfill
\begin{minipage}{\mylwd\linewidth}
  \centerline{\includegraphics[width=\mydmwd]{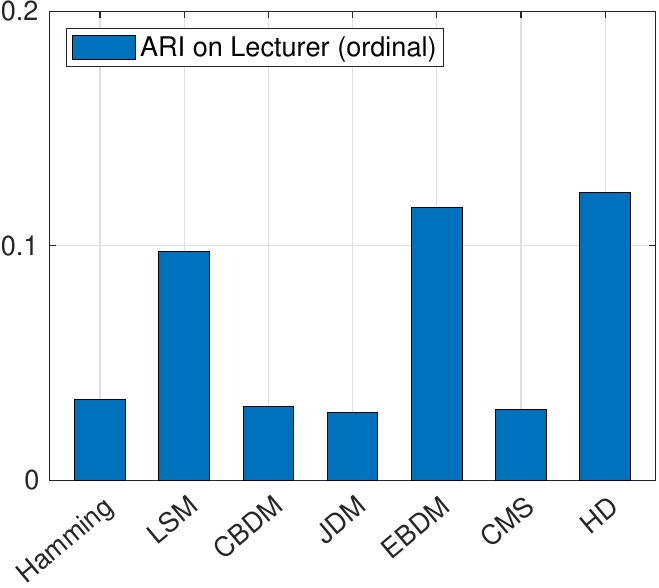}}
\end{minipage}
\hfill
\begin{minipage}{\mylwd\linewidth}
  \centerline{\includegraphics[width=\mydmwd]{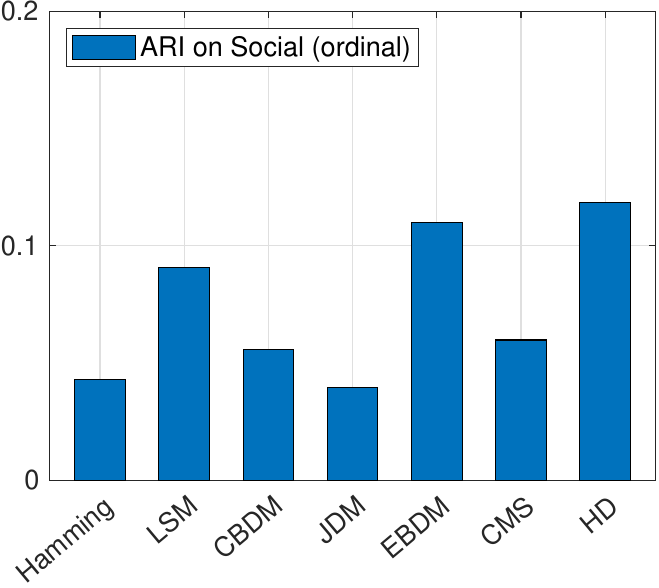}}
\end{minipage}
\hfill
\begin{minipage}{\mylwd\linewidth}
  \centerline{\includegraphics[width=\mydmwd]{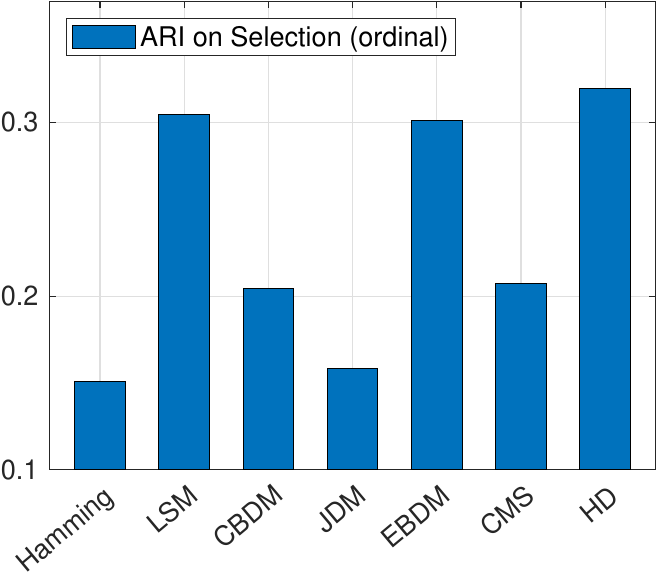}}
\end{minipage}
\vfill
\begin{minipage}{\mylwd\linewidth}
  \centerline{\includegraphics[width=\mydmwd]{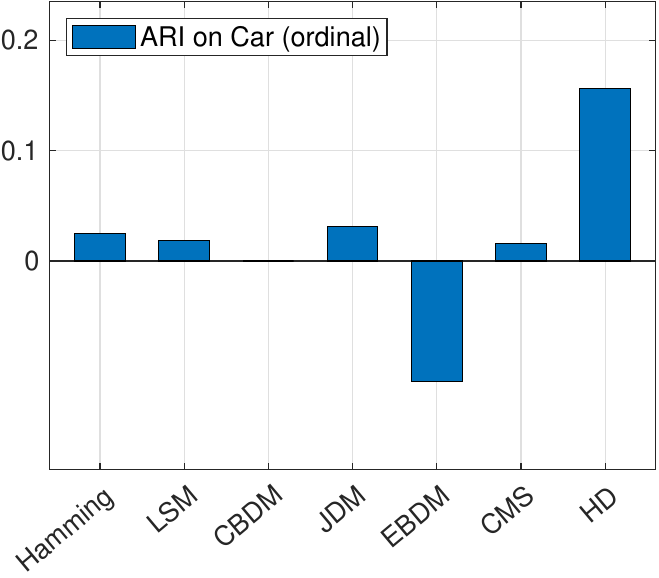}}
\end{minipage}
\hfill
\begin{minipage}{\mylwd\linewidth}
  \centerline{\includegraphics[width=\mydmwd]{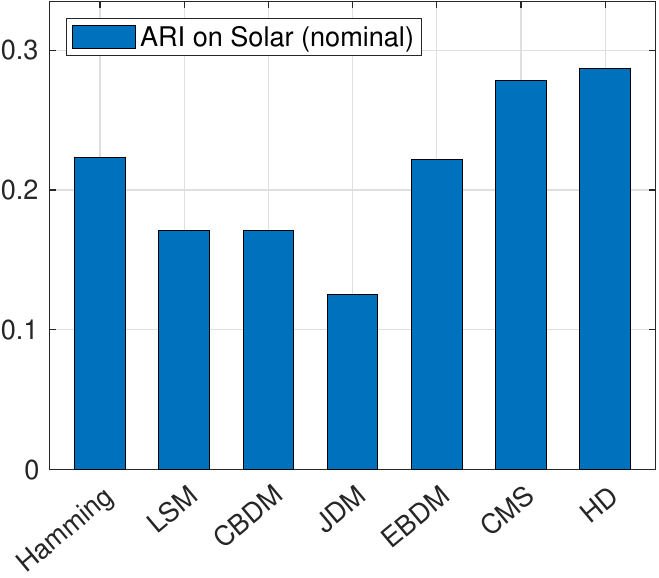}}
\end{minipage}
\hfill
\begin{minipage}{\mylwd\linewidth}
  \centerline{\includegraphics[width=\mydmwd]{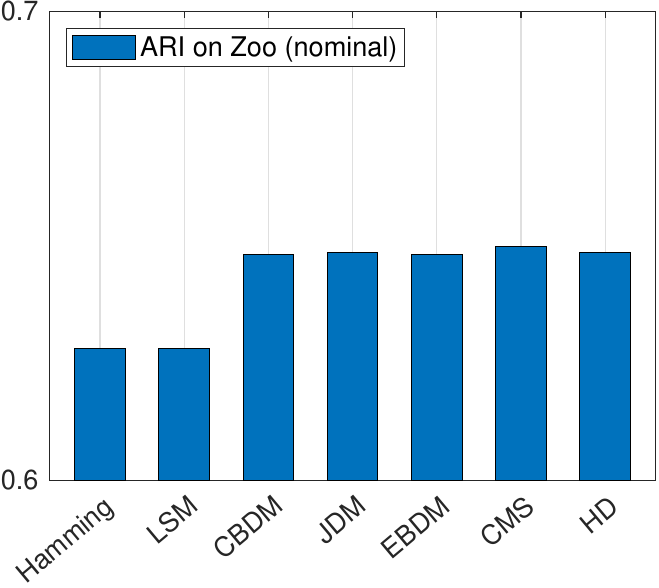}}
\end{minipage}
\hfill
\begin{minipage}{\mylwd\linewidth}
  \centerline{\includegraphics[width=\mydmwd]{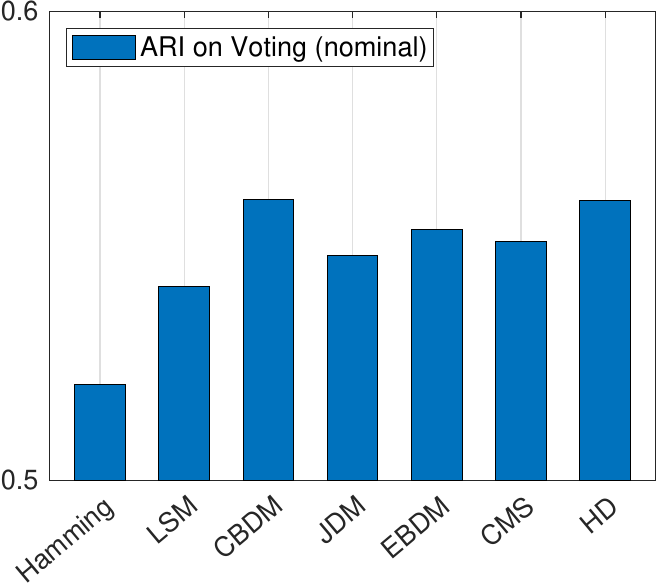}}
\end{minipage}
\hfill
\begin{minipage}{\mylwd\linewidth}
  \centerline{\includegraphics[width=\mydmwd]{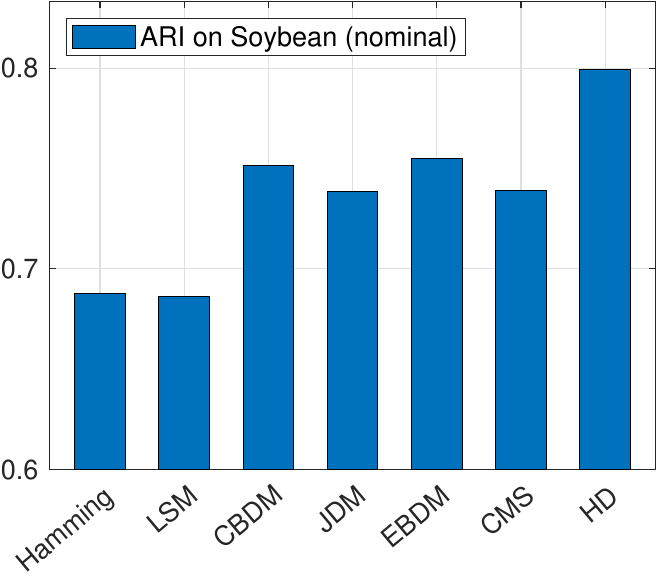}}
\end{minipage}
\caption{Clustering performance of various distance measures on mixed, ordinal, and nominal data sets, where a better measure yields a higher value.}
\label{fig:com_metric_ar}
\end{figure*}

Fig.~\ref{fig:com_metric_ar} compares clustering performance of different distance measures and illustrates that even not weighted by NDW, distance measured using HD is still very competent. More detailed observations are provided below:
\begin{itemize}
\item HD outperforms the other counterparts on mixed data sets because it is the only one that can measure intra-attribute distances of nominal and ordinal attributes in a homogeneous way. HD outperforms the other counterparts on ordinal data sets because it preserves the order relationship among ordered possible values.
\item On Zoo, Voting, Cancer, and Lym data sets, performance of HD is competitive but cannot be obviously better than the others. This may be because that the above-mentioned four data sets are composed of more nominal attributes, which weakens the advantages of HD accordingly.
\item ARI performance of CBDM is exactly 0 on Lenses, Nursery, and Car data sets because these data sets are composed of independent attributes and CBDM fails in measuring distances for such data sets.
\end{itemize}

\subsection{Effectiveness Evaluation of NDW}\label{subsct:ex_weighting}

\begin{figure}[t]
\newcommand{\mylwd}{0.325}
\newcommand{\mydmwd}{1.1in}
\begin{minipage}{\mylwd\linewidth}
  \centerline{\includegraphics[width=\mydmwd]{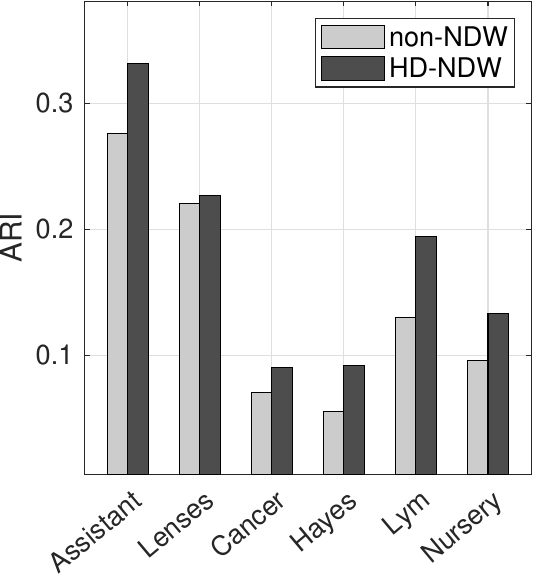}}
\end{minipage}
\hfill
\begin{minipage}{\mylwd\linewidth}
  \centerline{\includegraphics[width=\mydmwd]{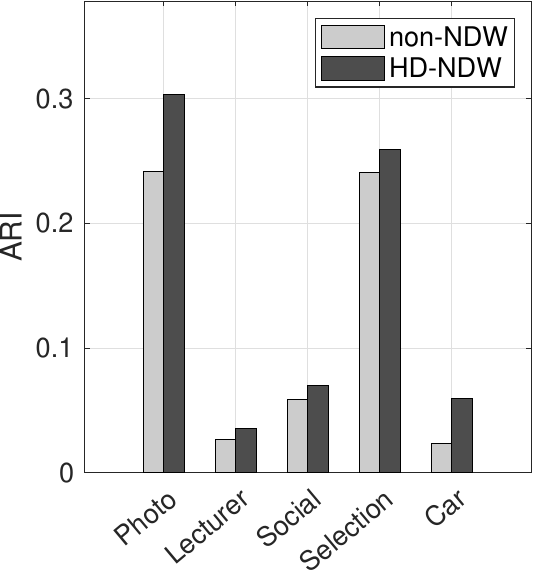}}
\end{minipage}
\hfill
\begin{minipage}{\mylwd\linewidth}
  \centerline{\includegraphics[width=\mydmwd]{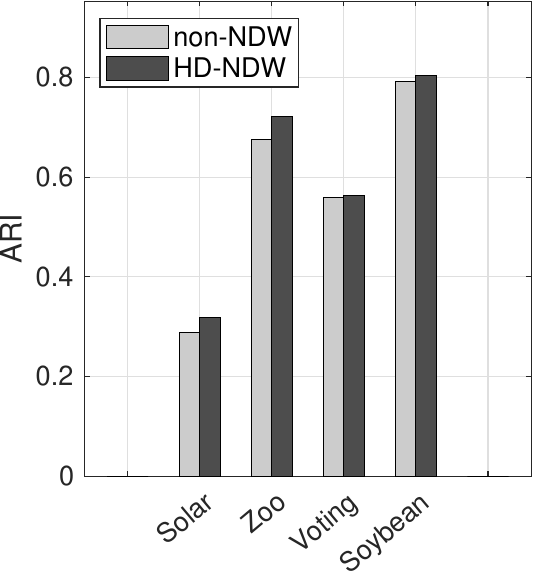}}
\end{minipage}
\caption{Clustering performance of HD-NDW and its version without NDW (non-NDW for short) on mixed, ordinal, and nominal data sets. A higher value indicates a better clustering performance.}
\label{fig:com_weight_ar}
\end{figure}

Clustering performance of the original version of HD-NDW and the version without NDW (abbreviated as non-NDW) is demonstrated in Fig.~\ref{fig:com_weight_ar}. By comparing them, effectiveness of the NDW mechanism can be empirically proved.

It can be observed that HD-NDW performs better than the non-NDW version on all the data sets, which indicates that the NDW mechanism does optimize the distance weights during the clustering of HD-NDW to obtain better clustering results. It can also be observed that HD-NDW does not outperform non-NDW a lot on Lenses, Voting and Soybean data sets. This may be because most attributes of these three data sets have only two possible values, and for such attributes, there is only one intra-attribute distance to be weighted during clustering, which makes NDW degrades into a conventional attribute weighting mechanism, and thus obscures the merits of NDW.

\subsection{Convergence Evaluation}\label{subsct:ex_convergence}

\begin{figure*}[t]
\newcommand{\mylwd}{0.19}
\newcommand{\mydmwd}{1.25in}
\begin{minipage}{\mylwd\linewidth}
  \centerline{\includegraphics[width=\mydmwd]{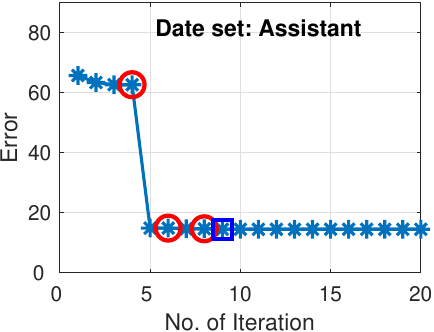}}
\end{minipage}
\hfill
\begin{minipage}{\mylwd\linewidth}
  \centerline{\includegraphics[width=\mydmwd]{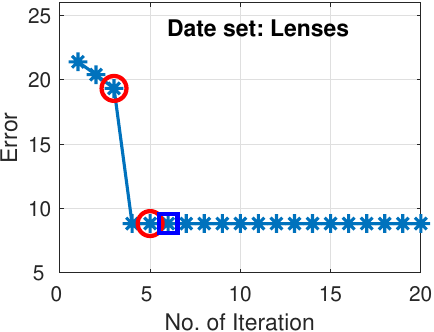}}
\end{minipage}
\hfill
\begin{minipage}{\mylwd\linewidth}
  \centerline{\includegraphics[width=\mydmwd]{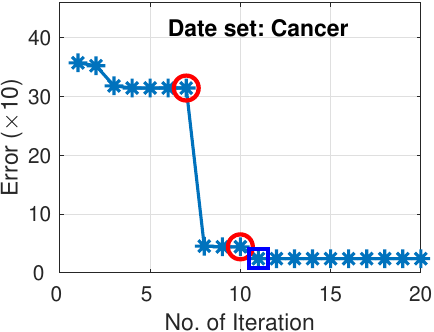}}
\end{minipage}
\hfill
\begin{minipage}{\mylwd\linewidth}
  \centerline{\includegraphics[width=\mydmwd]{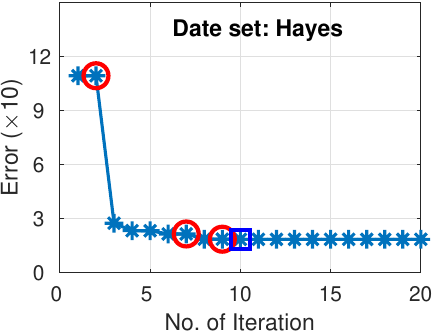}}
\end{minipage}
\hfill
\begin{minipage}{\mylwd\linewidth}
  \centerline{\includegraphics[width=\mydmwd]{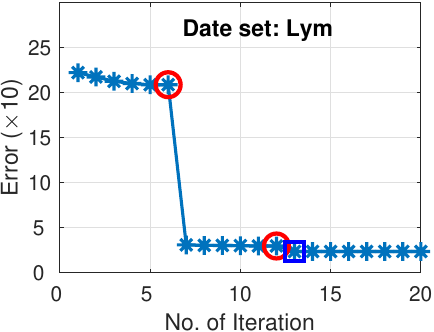}}
\end{minipage}
\vfill
\begin{minipage}{\mylwd\linewidth}
  \centerline{\includegraphics[width=\mydmwd]{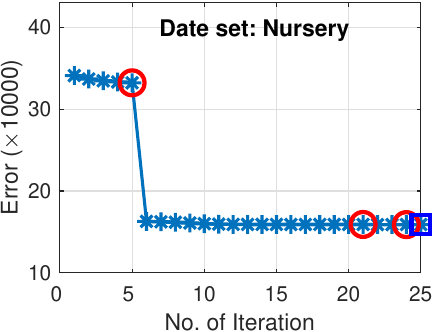}}
\end{minipage}
\hfill
\begin{minipage}{\mylwd\linewidth}
  \centerline{\includegraphics[width=\mydmwd]{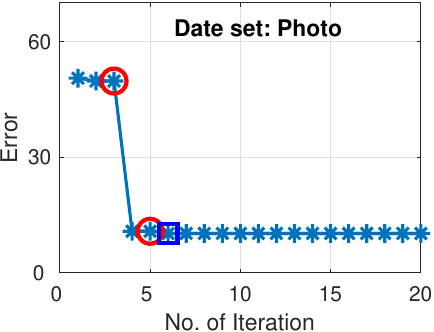}}
\end{minipage}
\hfill
\begin{minipage}{\mylwd\linewidth}
  \centerline{\includegraphics[width=\mydmwd]{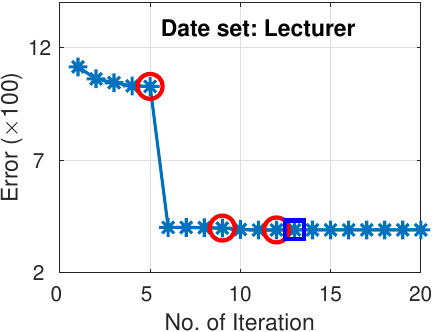}}
\end{minipage}
\hfill
\begin{minipage}{\mylwd\linewidth}
  \centerline{\includegraphics[width=\mydmwd]{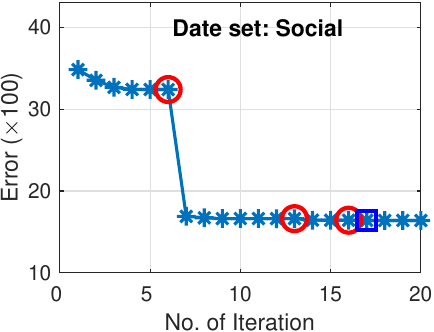}}
\end{minipage}
\hfill
\begin{minipage}{\mylwd\linewidth}
  \centerline{\includegraphics[width=\mydmwd]{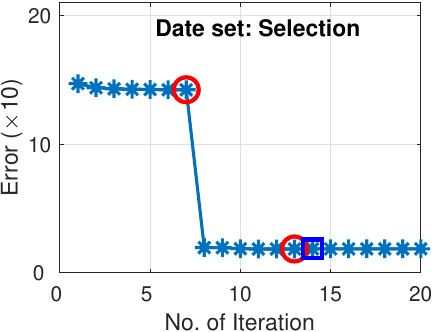}}
\end{minipage}
\vfill
\begin{minipage}{\mylwd\linewidth}
  \centerline{\includegraphics[width=\mydmwd]{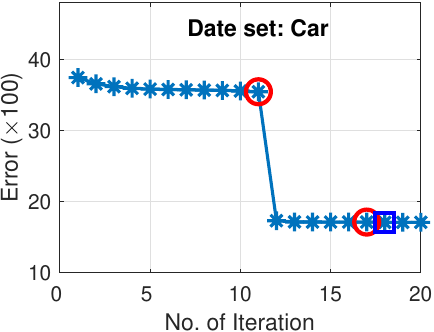}}
\end{minipage}
\hfill
\begin{minipage}{\mylwd\linewidth}
  \centerline{\includegraphics[width=\mydmwd]{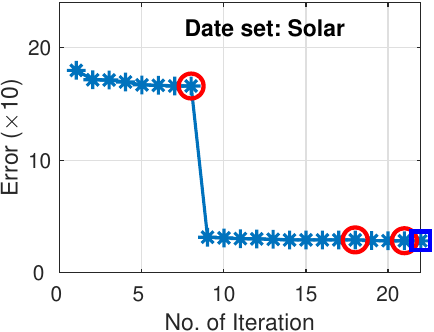}}
\end{minipage}
\hfill
\begin{minipage}{\mylwd\linewidth}
  \centerline{\includegraphics[width=\mydmwd]{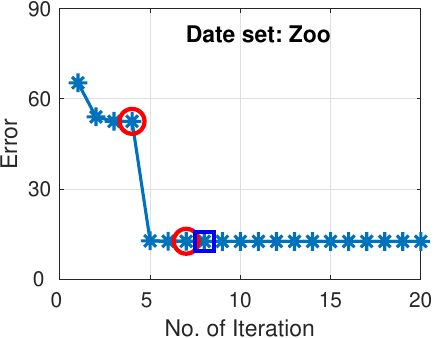}}
\end{minipage}
\hfill
\begin{minipage}{\mylwd\linewidth}
  \centerline{\includegraphics[width=\mydmwd]{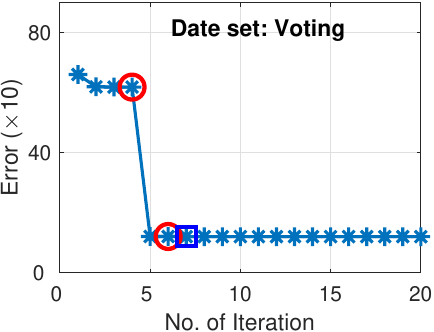}}
\end{minipage}
\hfill
\begin{minipage}{\mylwd\linewidth}
  \centerline{\includegraphics[width=\mydmwd]{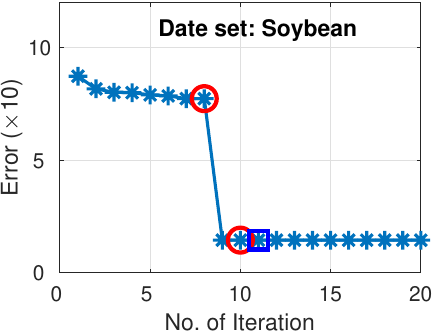}}
\end{minipage}
\caption{Convergence curves of HD-NDW on mixed, ordinal, and nominal data sets. The circles indicate the moments that Step 2 of Algorithm~\ref{alg:algorithm} is triggered, and the boxes indicate the moments of convergence of Algorithm~\ref{alg:algorithm}.}
\label{fig:conv}
\end{figure*}

We plot the convergence curves of HD-NDW on each data set in Fig.~\ref{fig:conv}. Specifically, after each iteration of \textbf{Step 1} in Algorithm~\ref{alg:algorithm}, `No. of Iteration' is added by 1, and the current `Error' (i.e., the current value of objective function) is plotted. When \textbf{Step 1} converges and \textbf{Step 2} is triggered, the current `Error' is marked by a circle. When the whole algorithm converges, the current `Error' is marked by a box.

It can be seen that HD-NDW converges within 6 - 22 iterations on different data sets, which is very fast for learning a large number (i.e. $\sum_r^dv^r(v^r-1)/2$) of intra-attribute distance weights. Moreover, the convergence curves are monotonically decreasing, and `Error' decreases sharply after updating the distance weights, which clearly illustrates the effectiveness of HD-NDW.

In our experiments, since the true number $k^*$ of the clusters is utilized, partition learned by \textbf{Step 1} is relatively reasonable, which offers useful information for learning $W$ in \textbf{Step 2}. This would be the reason why \textbf{Step 2} is always triggered 2 - 3 times for different data sets.

\subsection{Computational Efficiency Evaluation}\label{subsct:ex_efficiency}

\begin{figure}[t]
\newcommand{\mylwd}{0.48}
\newcommand{\mydmwd}{1.5in}
\begin{minipage}{\mylwd\linewidth}
  \centerline{\includegraphics[width=\mydmwd]{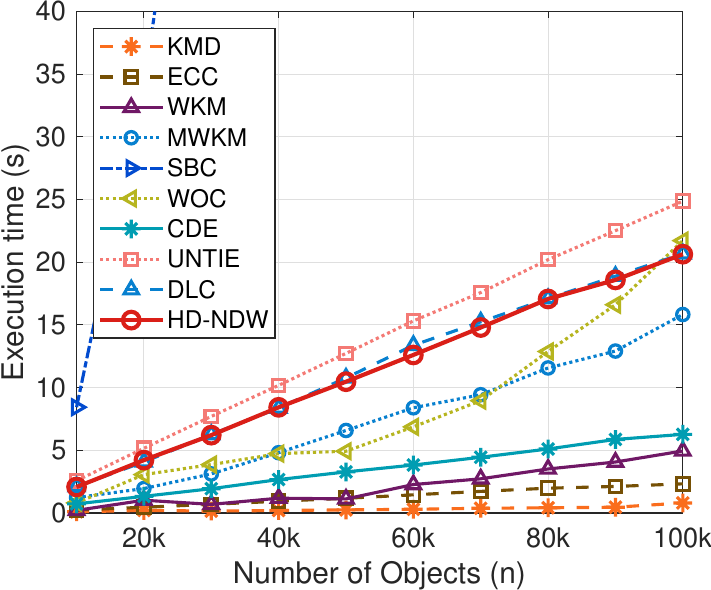}}
\end{minipage}
\hfill
\begin{minipage}{\mylwd\linewidth}
  \centerline{\includegraphics[width=\mydmwd]{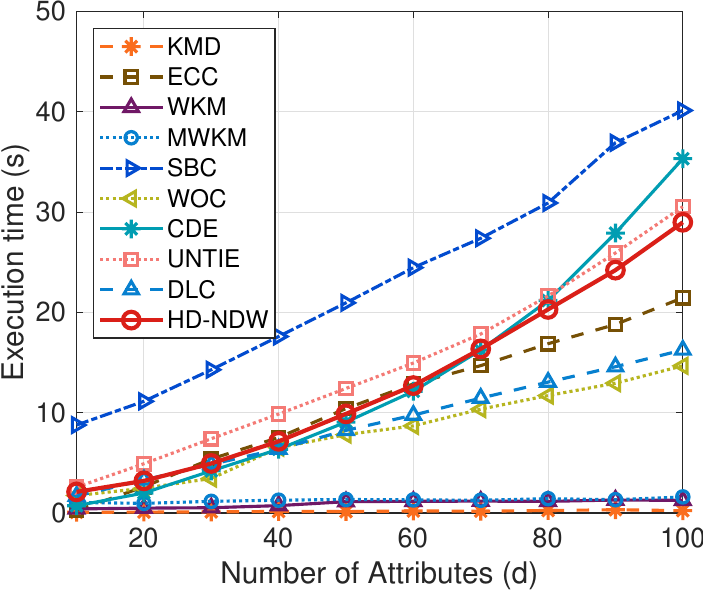}}
\end{minipage}
\vfill
\begin{minipage}{\mylwd\linewidth}
  \centerline{\includegraphics[width=\mydmwd]{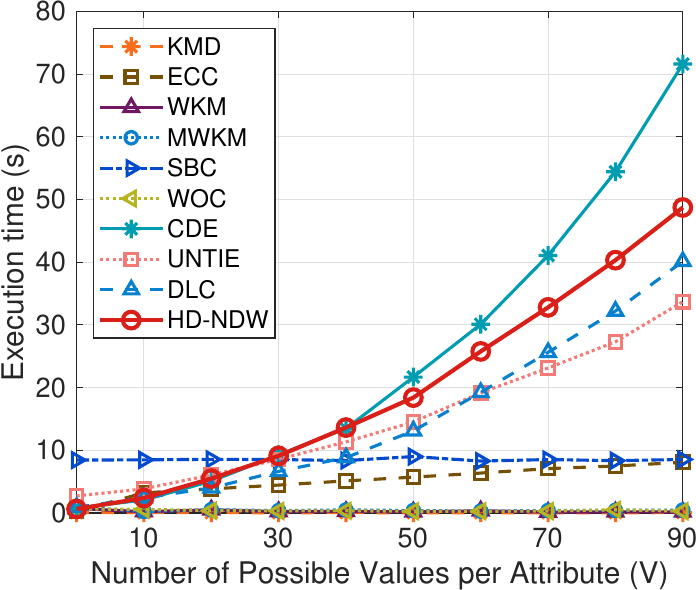}}
\end{minipage}
\hfill
\begin{minipage}{\mylwd\linewidth}
  \centerline{\includegraphics[width=\mydmwd]{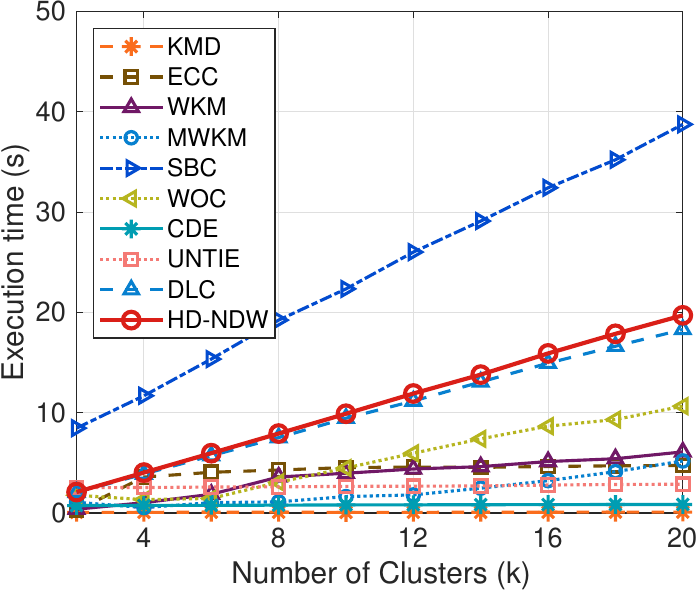}}
\end{minipage}
\caption{Execution time of various clustering algorithms w.r.t. number of objects ($n$), number of attributes ($d$), number of possible values per attribute ($V$), and number of clusters ($k$).}
\label{fig:com_time}
\end{figure}

We randomly generate synthetic categorical data sets to evaluate the computational efficiency of different clustering methods in terms of four data factors: (1) number of data objects ($n$), (2) number of attributes ($d$), (3) number of possible values per attribute ($V$), and number of clusters ($k$). Synthetic data sets are generated by increasing the value of one factor and fixing the other three factors at the default values. The default values are set at $n=10k$, $d=10$, $V=3$, and $k=2$. The value ranges for increasing each factor are set at $n=\{10k, 20k, ..., 100k\}$, $d=\{10, 20, ..., 100\}$, $V=\{3, 10, 20, ..., 90\}$, and $k=\{2, 4, ..., 20\}$. As HD-NDW is proposed for mixed categorical data clustering, we let it treat each generated data set as comprising $d/2$ nominal and $d/2$ ordinal attributes in this experiment. Since the data representation learning of SBC, CDE, UNTIE, and the distance computation of HD-NDW are necessarily processed for clustering, their execution time is counted in for comparison. We plot the execution time of different methods in terms of the four data factors in Fig.~\ref{fig:com_time}. It can be observed that the computation cost of HD-NDW has approximately linear relation with $n$ and $k$, which are consistent with the time complexity analysis in Section~\ref{subsct:dist} and~\ref{subsct:algorithm}.

In comparison with the state-of-the-art methods (i.e., UNTIE and DLC), it can be observed that the trends and values of the computation cost of HD-NDW and DLC are almost the same in terms of $n$ and $k$. Furthermore, HD-NDW has lower computation cost than UNTIE in terms of $n$. The computation cost of HD-NDW and DLC has higher increasing rate than UNTIE over $k$, because HD-NDW and DLC connect the distance learning with the target clustering task, and thus have better clustering performance in general as shown in Table~\ref{tb:comalgmixord}. Since $k$ is usually a very small value from the practical point of view, $k$ will not have a big impact on the efficiency of HD-NDW. Moreover, although the computation cost of HD-NDW has higher increasing rate over $d$ and $V$, the computations (e.g., the computation of each value in $D$, and the computation of each value in $W$) that are related to these two factors are independent and can be easily parallelized for acceleration.

In summary, HD-NDW does not bring much extra computation cost in comparison with the state-of-the-art methods, and its computation cost has a linear relation with $n$, which is generally the most concerned factor in terms of the computational efficiency of a clustering method.

\section{Conclusion}\label{sct:conclusion}

In this paper, we have proposed HD intra-attribute distance definition and NDW distance weighting mechanism, both of which are utilized to present HD-NDW clustering algorithm for data clustering with nominal and ordinal attributes. HD is formed based on the intrinsic connection of ordinal and nominal attributes, and can therefore define their intra-attribute distances in a homogeneous way. In the clustering process of HD-NDW, NDW novelly quantifies and iteratively updates the weights of intra-attribute distances defined by HD according to the present data partition, thereby ensuring an effective learning of the importance of intra-attribute distances for searching optimal clustering results. It turns out that HD-NDW is capable of clustering categorical data composed of any combination of nominal and ordinal attributes. Extensive experimental results have demonstrated that HD-NDW always converges quickly and has superior clustering performance in comparison with the existing counterparts.


\bibliographystyle{IEEEtran}
\bibliography{TPAMI_Mix_Clustering_Zyq}



\end{document}